\documentclass[shortAfour,sageh,times]{sagej}

\usepackage{moreverb,url}

\newcommand\BibTeX{{\rmfamily B\kern-.05em \textsc{i\kern-.025em b}\kern-.08em
T\kern-.1667em\lower.7ex\hbox{E}\kern-.125emX}}

\usepackage{units}
\usepackage{mathtools}
\usepackage{enumitem}

\usepackage{dsfont}

\usepackage[caption=false]{subfig}
\usepackage[titlenumbered,ruled,linesnumbered]{algorithm2e}

\usepackage[export]{adjustbox}

\usepackage{thmtools}
\usepackage{thm-restate}
\usepackage{color}

\usepackage[colorlinks=true,bookmarksopen,bookmarksnumbered,citecolor=black,urlcolor=black]{hyperref}
\usepackage[capitalise,noabbrev]{cleveref}

\usepackage{enumitem} % enumerate with roman letters

\crefname{equation}{}{} % Default to eqref for equations

\defcitealias{TheGPyauthors2012GPy}{GPy 2012}

\theoremstyle{plain}
\newtheorem{theorem}{Theorem}
\newtheorem{lemma}{Lemma}

\newtheorem{corollary}{Corollary}

\theoremstyle{definition}
\newtheorem{assumption}{Assumption}

\theoremstyle{remark}

\newcommand{\T}{\mathrm{T}}             % Transpose
\newcommand{\mb}[1]{\mathbf{#1}}        % Bold font for variables
\newcommand{\ts}[1]{_{#1}}           % Time step subscript

\newcommand{\ti}{n}

\newcommand{\LLC}{$L$-Lipschitz continuous}
\newcommand{\Reps}{R_{\epsilon}}
\newcommand{\Rbeps}{\bar{R}_{\epsilon}}

\newcommand{\Rbo}{\bar{R}_{0}}
\newcommand{\N}{N_n}

\def\*#1{\bm{#1}}
\newcommand{\defeq}{\vcentcolon=}

\newcommand{\argmax}{\operatornamewithlimits{argmax}}

\graphicspath{{img/}}

\begin{document}

\title{Bayesian Optimization with Safety Constraints: Safe and Automatic Parameter Tuning in Robotics}
\author{Felix Berkenkamp\affilnum{1},
Andreas Krause\affilnum{1}, and
Angela P. Schoellig\affilnum{2}}
\affiliation{\affilnum{1}Learning \& Adaptive Systems Group, Department of Computer Science, ETH Zurich, Switzerland.\\
\affilnum{2}Dynamic Systems Lab, Institute for Aerospace Studies, University of Toronto, Canada.}
\corrauth{Felix Berkenkamp, Institute for Machine Learning, CAB G 66, Universitaetstrasse 6, 8092 Zurich, Switzerland.}
\email{befelix@inf.ethz.ch}
\runninghead{Berkenkamp et al}

%!TEX root = ../root.tex
%
\begin{abstract}
Robotic algorithms typically depend on various parameters, the choice of which significantly affects the robot's performance. While an initial guess for the parameters may be obtained from dynamic models of the robot, parameters are usually tuned manually on the real system to achieve the best performance. Optimization algorithms, such as Bayesian optimization, have been used to automate this process. However, these methods may evaluate unsafe parameters during the optimization process that lead to safety-critical system failures. Recently, a safe Bayesian optimization algorithm, called~\textsc{SafeOpt}, has been developed, which guarantees that the performance of the system never falls below a critical value; that is, safety is defined based on the performance function. However, coupling performance and safety is often not desirable in robotics. For example, high-gain controllers might achieve low average tracking error (performance), but can overshoot and violate input constraints. In this paper, we present a generalized algorithm that allows for multiple safety constraints separate from the objective. Given an initial set of safe parameters, the algorithm maximizes performance but only evaluates parameters that satisfy safety for all constraints with high probability. To this end, it carefully explores the parameter space by exploiting regularity assumptions in terms of a Gaussian process prior. Moreover, we show how context variables can be used to safely transfer knowledge to new situations and tasks. We provide a theoretical analysis and demonstrate that the proposed algorithm enables fast, automatic, and safe optimization of tuning parameters in experiments on a quadrotor vehicle. 
\end{abstract}

% \keywords{Bayesian Optimization, Safe Reinforcement Learning, Constrained Policy Search, Robotics, Quadrotors, Unmanned Aerial Vehicles}
\maketitle
%!TEX root = ../root.tex

% \begin{sm}
% \fb{upload circle videos on youtube}
% A video demonstrating the proposed safe, automatic controller optimization on a quadrotor vehicle can be found at \mbox{\url{http://tiny.cc/icra16_video}}.
% A Python implementation of the algorithm is available at \mbox{\url{https://github.com/befelix/SafeOpt}}.
% \end{sm}

\section{Introduction}

\begin{figure}[t!]
\includegraphics[width=\columnwidth]{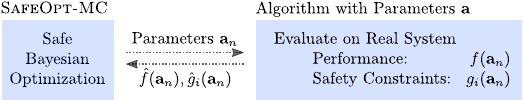}
\caption{Overview of the algorithm. At each iteration~$n$, the algorithm selects safe and informative parameters at which the performance and the safety constraints are evaluated on the real robot. Based on the noisy information gained, the algorithm updates its belief over the functions. This safe optimization process is iterated until the safely reachable, optimal parameters are found.}
\label{fig:optimization_illustration}
\end{figure}

Safety and the ability to operate within constraints imposed by an environment are critical prerequisites for any algorithm that is applied on a real robotic system. Especially in robotics, where systems often face large prior uncertainties, failures can cause serious damage to the robot and its environment~\citep{Schaal2010Learning}. To avoid unsafe behavior, safety is typically guaranteed with respect to a model of the robot's dynamics and environment. When accurate models are not available or when the robotic system contains elements that are difficult to model, such as computer vision components, the parameters of the algorithms are either tuned manually in experiments on the real system or tuned based on massive amounts of experimental data~\citep{Lillicrap2015Continuous}. Both methods are time-consuming and potentially safety-critical: the engineer must either carefully select parameters that are safe or collect enough representative data that leads to safe behavior.

In this paper, we present a method to automatically optimize parameters of robotics algorithms while respecting safety constraints during the optimization. The resulting algorithm can be used to optimize parameters on the real robot without failures, since no unsafe parameters are evaluated during the optimization.
We expand the theoretical framework of~\textsc{SafeOpt} (Safe Optimization) by~\cite{Sui2015Safe} to the more general setting with multiple constraints. We show that our algorithm,~\textsc{SafeOpt-MC}~(for multiple constraints), enjoys strong theoretical guarantees about safety and performance, and works well in practice.

\paragraph{Related work}
In control theory, guaranteeing safety in the presence of unmodeled dynamics is often interpreted as the problem of ensuring stability of  the  underlying
control  law  with  respect  to  an  uncertain dynamic system model~\citep{Zhou1998Essentials}. In this setting, controllers can be gradually improved by estimating the unmodeled dynamics and updating the control law based on this estimate. Safety can be guaranteed by ensuring that either the controller is robustly stable for all possible models within the uncertainty specification~\citep{Berkenkamp2015Safe,Berkenkamp2017Safe} or the system never leaves a safe subset of the state space~\citep{Ostafew2016Robust,Aswani2013Provably,Akametalu2014Reachability,Moldovan2012Safe,Turchetta2016Safe}. Both methods require a system model and uncertainty specification to be known~\textit{a priori}, which must be accurate enough to guarantee stability. In contrast, in our setting we do not assume to have access to a model of the system, but aim to directly optimize the parameters of a control algorithm, without violating safety constraints.

In the robotics literature, optimization algorithms have previously been applied with the goal of maximizing a user-defined performance function through iterative experiments. This is especially powerful when no prior model of the robot is available. However, typical algorithms in the literature do not consider safety of the optimization process, and make other restrictive assumptions such as requiring gradients~\citep{Killingsworth2006PID,Astrom1993Automatic}, which are difficult to obtain from noisy data, or an impractical number of experiments~\citep{Davidor1991Genetic}.

The objective of learning optimal policies has been extensively studied in the reinforcement learning community~\citep{Sutton1998Reinforcement}. In particular, the area of policy search considers the objective of optimizing the parameters of control algorithms~\citep{Kober2014Reinforcement}. The state of the art methods are based on estimating the gradients of the performance function~\citep{Peters2006Policy,Peters2008Reinforcement}. As a result, typically multiple evaluations of very similar parameters are conducted on the real system in order to estimate the gradients, which means the approaches are often not data-efficient and converge to local optima. Safety in gradient-based policy search has only been considered by disallowing large steps along the gradient into areas of the parameter space that have not been explored before~\citep{Achiam2017Constrained}. Guarantees there hold only in expectation. In contrast, our method is gradient-free, so that it can explore the parameter space globally in a more data-efficient manner. At the same time, we provide high-probability worst-case guarantees for not violating safety constraints during the optimization process.

One class of optimization algorithms that has been successfully applied to robotics is Bayesian optimization~\citep{Mockus2012Bayesian}. In Bayesian optimization, rather than considering the objective function as a black-box about which we can only obtain point-wise information, regularity assumptions are made. These are used to actively learn a model of the objective function. The resulting algorithms are practical and provably find the global optimum of the objective function while evaluating the function at only few parameters~\citep{Bull2011Convergence,Srinivas2012Gaussian}. Bayesian optimization methods often model the unknown function as a Gaussian process (GP)~\citep{Rasmussen2006Gaussian}. These models are highly flexible, allow to encode as much prior knowledge as desired, and explicitly model noise in the performance function evaluations. The GP models are used to guide function evaluations to locations that are informative about the optimum of the unknown function~\citep{Mockus2012Bayesian,Jones2001Taxonomy}.
Example applications of Bayesian optimization in robotics include gait optimization of legged robots~\citep{Calandra2014Bayesian,Lizotte2007Automatic} and the optimization of the controller parameters of a snake-like robot~\citep{Tesch2011Using}. \citet{Marco2017Virtual} optimize the weighting matrices of an LQR controller for an inverted pendulum by exploiting additional information from a simulator. Several different Bayesian optimization methods are compared by~\cite{Calandra2014experimental} for the case of bipedal locomotion. While these examples illustrate the potential of Bayesian optimization methods in robotics, none of these examples explicitly considers safety as a requirement.

Recently, the concept of constraints has been incorporated into Bayesian optimization. \cite{Gelbart2014Bayesian} introduce an algorithm to optimize an unknown function subject to an unknown constraint. However, this constraint was not considered to be safety-critical; that is, the algorithm is allowed to evaluate unsafe parameters. The case of finding a safe subset of the parameters without violating safety constraints was considered by~\cite{Schreiter2015Safe}, while~\cite{Sui2015Safe} presented a similar algorithm to safely optimize an objective function. However, the algorithm of~\cite{Sui2015Safe} considers safety as a minimum performance requirement. In robotics, safety constraints are typically functions of the states or inputs that are independent of the performance.

\paragraph{Our contributions}
\looseness -1 In this paper, we present an algorithm that considers multiple, arbitrary safety constraints decoupled from the performance objective. This generalization retains the desirable sample-efficient properties of normal Bayesian optimization, but carefully explores the parameter space in order to maximize performance while guaranteeing safety with high probability. We extend the theory of~\textsc{SafeOpt}~\citep{Sui2015Safe} to account for these additional constraints and show that similar theoretical guarantees can be obtained for the more general setting. We then relax the assumptions used in the proofs to obtain a more practical version of the algorithm and additionally show that the safety guarantees carry over to this case. Next to the theoretical contributions, the second main contribution is an {\em extensive experimental evaluation} of the method, where we consider the problem of safely optimizing linear and nonlinear laws on a quadrotor vehicle. The experiments demonstrate that the proposed approach is able to safely optimize parameters of a control algorithms while respecting safety constraints with high probability. Moreover, we show how ideas from context-based optimization~\citep{Krause2011Contextual} can be used to safely transfer knowledge in order to obtain environment-dependent control laws. For example, in our experiments we optimize a control law for different flying speeds of a quadrotor vehicle. Early results with only a safety constraint on performance and without additional theoretical results were presented in~\citep{Berkenkamp2016Safe}.

The rest of the paper is structured as follows: in~\cref{sec:problem_statement}, we define the problem of safely optimizing the parameters of control algorithms and provide an overview on GPs and Bayesian optimization in~\cref{sec:background}. In~\cref{sec:safe_opt_adapted}, we introduce our algorithm, and analyze its theoretical properties in~\cref{sec:theoretical_results}. We evaluate the performance of the algorithm in experiments on a quadrotor vehicle in~\cref{sec:results} and draw conclusions in~\cref{sec:conclusion}. The proofs of the theorems are provided in~\cref{sec:proofs}.

%!TEX root = ../root.tex

\section{Problem Statement}
\label{sec:problem_statement}

We consider a given algorithm that is used to accomplish a certain task with a robot. In general, this algorithm is arbitrary and may contain several components including vision, state estimation, planning, and control laws. The algorithm depends on tuning parameters~${\mb{a} \in \mathcal{A} }$ in some specified, domain~$\mathcal{A} \subseteq \mathbb{R}^d$.

The goal is to find the parameters within~$\mathcal{A}$ that maximize a given, scalar performance measure,~$f$. For example, this performance measure may represent the negative tracking error of a robot~\citep{Berkenkamp2016Safe}, the average walking speed of a bipedal robot~\citep{Calandra2014Bayesian}, or any other quantity that can be computed over a finite time horizon. We can only evaluate the performance measure for any parameter set~$\mb{a}$ on finite-time trajectories from experiments on the real robot. The functional dependence of~$f$ on~$\mb{a}$ is not known \textit{a priori}. In the following, we write the performance measure as a function of the parameters~$\mb{a}$,~${ f \colon \mathcal{A} \to \mathbb{R} }$, even though measuring performance requires an experiment on the physical robot and typically depends on a trajectory of states, control inputs, and external signals.

We assume that the underlying system is safety-critical; that is, there are constraints that the system must respect when evaluating parameters. Similarly to the performance measure,~${ f(\mb{a}) }$, these constraints can represent any quantity and may depend on states, inputs, or even environment variables. There are~$q$ safety constraints of the form~${ g_i(\mb{a}) \geq 0,\, g_i\colon \mathcal{A} \to \mathbb{R}, \,i=1 \dots q}$, which together define the safety conditions. This is without loss of generality, since any constraint function can be shifted by a constant in order to obtain this form. The functions~$g_i$ are unknown~\textit{a priori} but can be estimated through (typically noisy) experiments for a given parameter set~$\mb{a}$. For example, in order to encode a state constraint on an obstacle for a robot, the safety function~$g_i(\mb{a})$ can return the smallest distance to the obstacle along a trajectory of states when using algorithm parameters~$\mb{a}$. Note that if the functions were known in advance, we could simply exclude unsafe parameters from the set~$\mathcal{A}$.

The overall optimization problem can be written as
\begin{equation}
\max_{\mb{a} \in \mathcal{A}} f(\mb{a}) ~~\textnormal{subject to}~~ g_i(\mb{a}) \geq 0 \, \forall \, i=1,\dots,q \textnormal{.}
\label{math:global_optimization_problem}
\end{equation}
The goal is to iteratively find the global maximum of this constrained optimization problem by, at each iteration~$n$, selecting parameters~$\mb{a}_n$ and evaluating (up to noise) the corresponding function values~$f(\mb{a}_n)$ and $g_i(\mb{a}_n)$ until the optimal parameters are found. In particular, since the constraints define the safety of the underlying system, only parameters that are inside the feasible region of~\cref{math:global_optimization_problem} are allowed to be evaluated; that is, only parameters that fulfill these safety requirements on the real system.

Since the functions~$f$ and~$g_i$ in~\cref{math:global_optimization_problem} are unknown \textit{a priori}, it is not generally possible to solve the corresponding optimization problem without violating the constraints. The first problem is that we do not know how to select a first, safe parameter to evaluate. In the following, we assume that an initial safe set of parameters~${S_0 \subseteq \mathcal{A}}$ is known for which the constraints are fulfilled. These serve as a starting point for the exploration of the safe region in~\cref{math:global_optimization_problem}. In robotics, safe initial parameters with poor performance can often be obtained from a simulation or domain knowledge.

Secondly, in order to safely explore the parameter space beyond~$S_0$, we must be able to infer whether parameters~$\mb{a}$ that we have not evaluated yet are safe to use on the real system. To this end, we make regularity assumptions about the functions~$f$ and~$g_i$ in~\cref{math:global_optimization_problem}. We discuss these assumptions in more detail in~\cref{sec:theoretical_results}. However, broadly speaking we make assumptions that allow us to model the functions~$f$ and~$g_i$ as a GP, construct reliable confidence intervals over the domain~$\mathcal{A}$, and imply Lipschitz continuity properties. Using these properties, we are able to generalize safety beyond the initial, safe parameters~$S_0$. Given the model assumptions, we require that the safety constraints hold with high probability over the entire sequence of experiments.

As a consequence of the safety requirements, it is not generally possible to find the global optimum of~\cref{math:global_optimization_problem}. Instead we aim to find the optimum in the part of the feasible region that is safely reachable from~$S_0$. We formalize this precisely in~\cref{sec:safe_opt_adapted}.

Lastly, whenever we evaluate parameters on the real system, we only obtain noisy estimates of both the performance function and the constraints, since both depend on noisy sensor data along trajectories. That is, for each parameter~$\mb{a}$ the we evaluate, we obtain measurements~${\hat{f}(\mb{a}) = f(\mb{a}) + \omega_0}$ and~${\hat{g}_i(\mb{a}) = g_i(\mb{a}) + \omega_{i}}$, where~${\omega_{i},\,i=0,\dots,q}$, is zero-mean, $\sigma$-sub-Gaussian noise. Note that while~$\hat{f}(\mb{a})$ is a random variable, we use~$\hat{f}(\mb{a}_n)$ to denote the measurement obtained at iteration~$n$. In general, the noise variables may be correlated, but we do not consider this case in our theoretical analysis in~\cref{sec:theoretical_results}. We only want to evaluate parameters where all safety constraints are fulfilled, so that~${g_i(\mb{a}_n) \geq 0}$ for all~$i \in \{1,\dots,q\}$ and~$n \geq 1$.

%!TEX root = ../root.tex

\section{Background}
\label{sec:methodology}
\label{sec:background}

In this section, we review Gaussian processes (GPs) and Bayesian optimization, which form the foundation of our safe Bayesian optimization algorithm in~\cref{sec:safe_opt_adapted}. The introduction to GPs is standard and based  on~\citep{Berkenkamp2016Safe} and~\citep{Rasmussen2006Gaussian}.

\subsection{Gaussian Process (GP)}
\label{sec:gaussian_process}

Both the function~$f(\mb{a})$ and the safety constraints~$g_i(\mb{a})$ in~\cref{sec:problem_statement} are unknown \textit{a priori}. We use GPs as a nonparametric model to approximate these unknown functions over their common domain~$\mathcal{A}$. In the following, we focus on a single function, the performance function. We extend this model to multiple functions in order to represent both performance and constraints in~\cref{sec:multi_gps}.

GPs are a popular choice for nonparametric regression in machine learning, where the goal is to find an approximation of a nonlinear map,~${f(\mb{a}): \mathcal{A} \to \mathbb{R}}$, from an input vector~${\mb{a} \in \mathcal{A} }$ to the function value~$f(\mb{a})$. This is accomplished by assuming that the function values $f(\mb{a})$, associated with different values of $\mb{a}$, are random variables and that any finite number of these random variables have a joint Gaussian distribution~\citep{Rasmussen2006Gaussian}.

A GP is parameterized by a prior mean function and a covariance function~$k(\mb{a}, \mb{a}')$, which defines the covariance of any two function values~$f(\mb{a})$ and $f(\mb{a}')$, ${\mb{a}, \mb{a}' \in \mathcal{A}}$.
The latter is also known as the kernel. In this work, the mean is assumed to be zero, without loss of generality. The choice of kernel function is problem-dependent and encodes assumptions about smoothness and rate of change of the unknown function. A review of potential kernels can be found in~\citep{Rasmussen2006Gaussian} and more information about the kernels used in this paper is given in~\cref{sec:results}.

The GP framework can be used to predict the function value~${f(\mb{a}^*)}$ for an arbitrary parameter~${ \mb{a}^* \in \mathcal{A} }$ based on a set of~$n$ past observations, ${ \{ \hat{f}(\mb{a}_i) \}_{i=1}^n}$, at the chosen parameters~${\mathcal{D}_n = \{ \mb{a}_i \}_{i=1}^n}$. The GP model assumes that observations are noisy measurements of the true function value~${f(\mb{a})}$; that is,~${\hat{f}(\mb{a}) = f(\mb{a}) + \omega}$ with ${\omega \sim \mathcal{N}(0,\sigma^2)}$. Conditioned on these observations, the posterior distribution is a GP again with mean and variance
\begin{align}
\mu_n(\mb{a}^*) &= \mb{k}_n(\mb{a}^*)    (\mb{K}_n + \mb{I}_n \sigma^2)^{-1} \hat{\mb{f}}_n ,
\label{math:gp_prediction_mean} \\
\sigma^2_n(\mb{a}^*) &= k(\mb{a}^*,\mb{a}^*) - \mb{k}_n(\mb{a}^*) (\mb{K}_n + \mb{I} \sigma^2)^{-1} \mb{k}_n^\T(\mb{a}^*),
\label{math:gp_prediction_variance}
\end{align}
where~${\hat{\mb{f}}_n = \left[ \begin{matrix}
\hat{f}(\mb{a}_1),\dots,\hat{f}(\mb{a}_n)
\end{matrix} \right] ^\T}$ is the vector of observed, noisy function values,
the covariance matrix~${\mb{K}_n \in \mathbb{R}^{n \times n}}$ has entries ${[\mb{K}_n]_{(i,j)} = k(\mb{a}_i, \mb{a}_j)}$, ${i,j\in\{1,\dots,n\}}$, and
the vector
${\mb{k}_n(\mb{a}^*) =
\left[ \begin{matrix}
	k(\mb{a}^*,\mb{a}_1),\dots,k(\mb{a}^*,\mb{a}_n)
\end{matrix}  \right]}$
contains the covariances between the new input~$\mb{a}^*$ and the observed data points in~$\mathcal{D}_n$.
The matrix~${ \mb{I}_n \in \mathbb{R}^{n \times n} }$ denotes the identity matrix.

\subsubsection{GPs with multiple outputs}
\label{sec:multi_gps}

So far, we have focused on GPs that model a single scalar function. In order to model not only the performance,~$f(\mb{a})$, but also the safety constraints,~$g_i(\mb{a})$, we have to consider multiple, possibly correlated functions. In the GP literature, these are usually treated by considering a matrix of kernel functions, which models the correlation between different functions~\citep{Alvarez2012Kernels}. Here instead, we use an equivalent representation by considering a surrogate function,
\begin{equation}
h(\mb{a}, i) =
\begin{cases}
f(\mb{a})  \quad \textnormal{if $i=0$} \\
g_i(\mb{a})  \quad \textnormal{if $i \in \mathcal{I}_g$},
\end{cases}
\label{math:surrogate}
\end{equation}
which returns either the performance function or the individual safety constraints depending on the additional input~${i \in \mathcal{I}}$ with~${\mathcal{I}=\{0,\dots,q\}}$, where~${\mathcal{I}_g = \{1, \dots, q\} \subset \mathcal{I}}$ are the indices belonging to the constraints. The function~$h(\cdot, \cdot)$ is a single-output function and can be modeled as a GP with scalar output over the extended parameter space~$\mathcal{A} \times \mathcal{I}$. For example, the kernel for the performance function~$f(\mb{a})$ and one safety constraint~$g(\mb{a})$ may look like this:
\begin{align}
k((\mb{a}, i),\, &(\mb{a}', j)) = \begin{cases}
\delta_{ij}\, k_f(\mb{a}, \mb{a}') + k_{fg}(\mb{a}, \mb{a}') ~ \textnormal{ if $i=0$} \\
\delta_{ij}\, k_g(\mb{a}, \mb{a}') + k_{fg}(\mb{a}, \mb{a}') ~ \textnormal{ if $i=1$},
\end{cases} \notag
\end{align}
where~$\delta_{ij}$ is the Kronecker delta. This kernel models the functions~$f(\mb{a})$ and~$g(\mb{a})$ with independent kernels~$k_f$ and~$k_g$ respectively, but also introduces a covariance function~$k_{fg}$ that models similarities between the two function outputs. By extending the training data by the extra parameter~$i$, we can use the normal GP framework and predict function values and corresponding uncertainties using~\cref{math:gp_prediction_mean,math:gp_prediction_variance}. When observing the function values, the index~$i$ is added to the parameter set~$\mb{a}$ for each observation. Including noise parameters inside the kernel allows to model noise correlation between the individual functions.

Importantly, using this surrogate function rather than the framework of~\cite{Alvarez2012Kernels} enables us to lift theoretical results of~\cite{Sui2015Safe} to the more general case with multiple constraints and provide theoretical guarantees for our algorithm in~\cref{sec:theoretical_results}.

In the setting with multiple outputs, at every iteration~$n$, we obtain~${|\mathcal{I}| = q + 1}$ measurements; one for each function. For ease of notation, we continue to write~$\mu_n$ and $\sigma_n$, even though we have obtained~$n \cdot (q + 1)$ measurements at locations~$\mathcal{D}_n \times \mathcal{I}$ in the extended parameter space.

\subsection{Bayesian Optimization}
\label{sec:bayesian_optimization}

Bayesian optimization aims to find the global maximum of an unknown function~\citep{Mockus2012Bayesian}. The framework assumes that evaluating the function is expensive, while computational resources are relatively cheap. This fits our problem in~\cref{sec:problem_statement}, where each evaluation of the performance function corresponds to an experiment on the real system, which takes time and causes wear in the robotic system.

In general, Bayesian optimization models the objective function as a random function and uses this model to determine informative sample locations. A popular approach is to model the underlying function as a GP, see~\cref{sec:gaussian_process}.
GP-based methods use the posterior mean and variance predictions in~\cref{math:gp_prediction_mean,math:gp_prediction_variance} to compute the next sample location.
For example, according to the~\textsc{GP-UCB} (GP-Upper Confidence Bound) algorithm by~\cite{Srinivas2012Gaussian}, the next sample location is
\begin{equation}
\mb{a}_n = \underset{\mb{a} \in \mathcal{A}}{\mathrm{argmax}}~ \mu_{n-1}(\mb{a}) + \beta_n^{1/2} \sigma_{n-1}(\mb{a}),
\label{math:gp_ucb}
\end{equation}
where~${\beta_n}$ is an iteration-dependent scalar that reflects the confidence interval of the GP.
Intuitively,~\cref{math:gp_ucb} selects new evaluation points at locations where the upper bound of the confidence interval of the GP estimate is maximal.
Repeatedly evaluating the system at locations given by~\cref{math:gp_ucb} improves the mean estimate of the underlying function and decreases the uncertainty at candidate locations for the maximum, such that the global maximum is provably found eventually~\citep{Srinivas2012Gaussian}.

While~\cref{math:gp_ucb} is also an optimization problem, its solution does not require any evaluations on the real system and only uses the GP model. This reflects the assumption of cheap computational resources. In practice, Bayesian optimization typically focuses on low-dimensional problems. However, this can be scaled up by discovering a low-dimensional subspace of~$\mathcal{A}$ for Bayesian optimization~\citep{Djolonga2013HighDimensional,Wang2013Bayesian} or encoding additional structure in the kernel~\citep{Duvenaud2011Additive}.

\subsection{Contextual Bayesian Optimization}
\label{sec:contextual_bayesian_optimization}

Contextual Bayesian optimization is a conceptually straightforward extension of Bayesian optimization~\citep{Krause2011Contextual}. It enables optimization of functions that depend on additional, external variables, which are called contexts. For example, the performance of a robot may depend on its battery level or the weather conditions, both of which cannot be influenced directly. Alternatively, contexts can also represent different tasks that the robot has to solve, which are specified externally by a user. The idea is to include the functional dependence on the context in the GP model, but to consider them fixed when selecting the next parameters to evaluate.

For example, given a context~${\mb{z} \in \mathcal{Z}}$ that is fixed by the environment, we can model how the performance and constraint functions change with respect to different contexts by multiplying the kernel function~$k_a$ over the parameters, with another kernel~${k_z \colon \mathcal{Z} \times \mathcal{Z} \to \mathbb{R}}$ over the contexts,
\begin{equation}
k((\mb{a}, i, \mb{z}), (\mb{a}', i', \mb{z}')) = k_a((\mb{a}, i), (\mb{a}', i')) \cdot k_z(\mb{z}, \mb{z}').
\label{eq:context_kernel}
\end{equation}
This kernel structure implies that function values are correlated when both parameters and the contexts are similar. For example, we would expect selecting the same parameters~$\mb{a}$ for a control algorithm to lead to similar performance values if the context (e.g., the battery level) is similar.

Since contexts are not part of the optimization criterion, a modified version of~\cref{math:gp_ucb} has to be used. It was shown by~\cite{Krause2011Contextual} that an algorithm that evaluates the GP-UCB criterion given a fixed context~$\mb{z}_n$,
\begin{equation}
\mb{a}_n = \underset{\mb{a} \in \mathcal{A}}{\mathrm{argmax}}~ \mu_{n-1}(\mb{a}, \mb{z}_n) + \beta_n^{1/2} \sigma_{n-1}(\mb{a}, \mb{z}_n),
\label{eq:contextual_gp_ucb}
\end{equation}
enjoys similar convergence guarantees as normal Bayesian optimization in~\cref{sec:bayesian_optimization}. Specifically, after seeing a particular context often enough, the criterion~\cref{eq:contextual_gp_ucb} will query parameters that are close-to-optimal.

\subsection{Safe Bayesian Optimization (\textsc{SafeOpt})}
\label{sec:safe_opt}

In this paper, we extend the safe optimization algorithm \textsc{SafeOpt}~\citep{Sui2015Safe} to multiple constraints.
\textsc{SafeOpt} is a Bayesian optimization algorithm, see~\cref{sec:bayesian_optimization}.
However, instead of optimizing the underlying performance function~$f(\mb{a})$ globally, it restricts itself to a safe set of parameters that achieve a certain minimum performance with high probability. This safe set is not known initially, but is estimated after each function evaluation.
In this setting, the challenge is to find an appropriate evaluation strategy similar to~\cref{math:gp_ucb}, which at each iteration~$n$ not only aims to find the global maximum within the currently known safe set (exploitation), but also aims to increase the set of controllers that are known to be safe (exploration).
\textsc{SafeOpt} trades off between these two sets by choosing for the next experiment the parameters inside the safe set about whose performance we are the most uncertain.

%!TEX root = ../root.tex

\section{\textsc{SAFEOPT-MC} (Multiple Constraints)}
\label{sec:safe_opt_adapted}

In this section, we introduce the~\textsc{SafeOpt-MC} algorithm for multiple constraints and discuss its theoretical properties. The goal of the algorithm is to solve~\cref{math:global_optimization_problem} by evaluating different parameters from the domain~$\mathcal{A}$ without violating the safety constraints. To this end, any algorithm has to consider two important properties:
\begin{enumerate}[label=(\roman*)]
\item Expanding the region of the optimization problem that is known to be feasible or safe as much as possible without violating the constraints,
\item Finding the optimal parameters within the current safe set.
\end{enumerate}
For objective~$\textit{i)}$, we need quantify the size of the safe set. To do this in a tractable manner, we focus on finite sets~$\mathcal{A}$ in the following. However, heuristic extensions to continuous domains exist~\citep{Duivenvoorden2017Constrained}.

The theoretical guarantees of the algorithm rely on the continuity of the underlying function. Many commonly used kernels, such as the squared exponential
(Gaussian) kernel, encode Lipschitz-continuous functions with high probability~\citep{Ghosal2006Posterior}. We make more specific assumptions that ensure deterministic Lipschitz constants in~\cref{sec:theoretical_results}. For now, we assume that~$f(\mb{a})$ and~$g_i(\mb{a})$ are Lipschitz continuous  with Lipschitz constant~$L$ with respect to some norm
\footnote{The functions~$f$ and~$g_i$ can have different Lipschitz constants~$L_i$, but we assume a global Lipschitz constant for ease of notation. Additionally, the theoretical results transfer equivalently to the case of Lipschitz-continuity with respect to some metric.}.

Since we only observe noisy estimates of both the performance function and the constraints, we cannot expect to find the entire safe region encoded by the constraints within a finite number of evaluations. Instead, we follow~\cite{Sui2015Safe} and consider learning the safety constraint up to some accuracy~$\epsilon$. This assumption is equivalent to a minimum slack of~$\epsilon$ on the constraints in~\cref{math:global_optimization_problem}.

As mentioned in~\cref{sec:problem_statement}, we assume that we have access to initial, safe parameters~${S_0 \subseteq \mathcal{A}}$, for which we know that the safety constraints are satisfied~\textit{a priori}.
Starting from these initial parameters, we ask what the best that any safe optimization algorithm could hope to achieve is. In particular, if we knew the safety constraint functions~$g_i(\cdot)$ up to~$\epsilon$ accuracy within some safe set of parameters~$S$, we could exploit the continuity properties to expand the safe set to
\begin{equation}
\begin{aligned}
&R_\epsilon (S) \coloneqq S \cup \\
&\quad  \bigcap_{i \in \mathcal{I}_g} \left\{ \mb{a} \in \mathcal{A} \, | \, \exists \, \mb{a}' \in S: \,g_i(\mb{a}') - \epsilon - L \|\mb{a}\hspace{-0.2em} -\hspace{-0.2em} \mb{a}'\| \geq 0 \right\},
\end{aligned}
\label{math:baseline_reachable_set}
\end{equation}
where~$R_\epsilon(S)$ represents the number of parameters that can be classified as safe given that we know~$g$ up to $\epsilon$-error inside~$S$ and exploiting the Lipschitz continuity to generalize to new parameters outside of~$S$. The baseline that we compare against is the limit of repeatedly applying this operator on~$S_0$; that is, with~${R^n_\epsilon(S) = R_\epsilon(R^{n-1}_\epsilon(S)) }$ and ${R^1_\epsilon(S) = R_\epsilon(S)}$ the baseline is~${\bar{R}_\epsilon(S_0) \coloneqq \lim_{n \to \infty} R_\epsilon^n(S_0)}$. This set contains all the parameters in~$\mathcal{A}$ that could be classified as safe starting from~$S_0$ if we knew the function up to~$\epsilon$ error. This set does not include all the parameters that potentially fulfill the constraints in~\cref{math:global_optimization_problem}, but is the best we can do without violating the safety constraints. Hence the optimal value that we compare against is not the one in~\cref{math:global_optimization_problem}, but
\begin{equation}
f_\epsilon^* = \max_{\mb{a} \in \bar{R}_\epsilon(S_0)} f(\mb{a}),
\label{math:baseline_max_value}
\end{equation}
which is the maximum performance value over the set that we could hope to classify as safe starting from the initial safe set,~$S_0$.

\subsection{The Algorithm}
\label{sec:algorithm}

\begin{figure*}[t]
\centering
\subfloat{\includegraphics[scale=1]{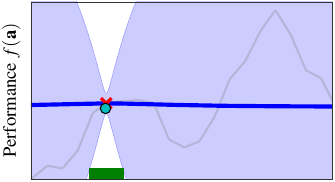} \label{fig:set_example_initial} } \hfill
\subfloat{\includegraphics[scale=1]{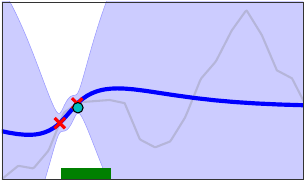} \label{fig:set_example_intermediate}} \hfill
\subfloat{\includegraphics[scale=1]{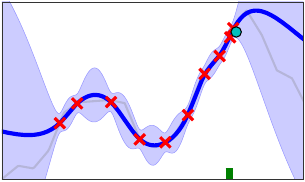} \label{fig:set_example_final}} \\
\subfloat[Initial, safe parameters.]{\includegraphics[scale=1]{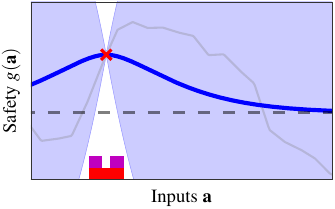} \label{fig:set_example_initial_safe} } \hfill
\subfloat[Safe exploration.]{\includegraphics[scale=1]{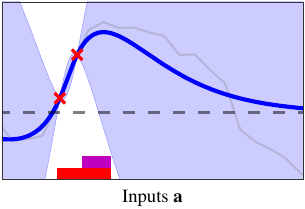} \label{fig:set_example_intermediate_safe}} \hfill
\subfloat[After 10 evaluations: safe maximum found.]{\includegraphics[scale=1]{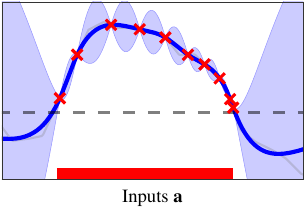} \label{fig:set_example_final_safe}}
\caption{Optimization with the \textsc{SafeOpt-MC} algorithm after 1, 2 and 10 parameter evaluations. Based on the mean estimate (blue) and the $2\sigma$ confidence interval (light blue), the algorithm selects evaluation points for which~$g(\mb{a}) \geq 0$ (black dashed) from the safe set~$S_n$ (red), which are either potential maximizers~$M_n$ (green) or expanders~$G_n$ (magenta). It then learns about the function by drawing noisy samples from the unknown, underlying function (light gray). This way, we expand the safe region (red) as much as possible and, simultaneously, find the global optimum of the unknown function~\cref{math:safeopt_pessimistic_estimate} (cyan circle).}
\label{fig:set_example}
\end{figure*}

In this section, we present the \textsc{SafeOpt-MC} algorithm that guarantees convergence to the previously set baseline. The most critical aspect of the algorithm is safety. However, once safety is ensured, the second challenge is to find an evaluation criterion that enables trading off between exploration, trying to further expand the current estimate of the safe set, and exploitation, trying to improving the estimate of the best parameters within the current set.

To ensure safety, we construct confidence intervals that contain the true functions~$f$ and~$g_i$ with high probability. In particular, we use the posterior GP estimate given the data observed so far. The confidence intervals for the surrogate function in~\cref{math:surrogate} are defined as
\begin{equation}
Q_n(\mb{a}, i) \coloneqq \left[ \mu_{n-1}(\mb{a}, i) \pm \beta_n^{1/2} \sigma_{n-1}(\mb{a}, i) \right],
\label{math:gp_confidence_intervals}
\end{equation}
where~$\beta_n$ is a scalar that determines the desired confidence interval. This set contains all possible function values between the lower and upper confidence interval based on the GP posterior. The probability of the true function value lying within this interval depends on the choice of~$\beta_n$, as well as on the assumptions made about the functions. We provide more details about this choice in~\cref{sec:theoretical_results},~\cref{thm:confidence_interval}, and~\cref{sec:practical_algorithm}.

Rather than defining the lower and upper bounds based on~\cref{math:gp_confidence_intervals}, the following analysis requires that consecutive estimates of the lower and upper bounds are contained within each other. This assumption ensures that the safe set does not shrink from one iteration to the next, which we require to prove our results. We relax this assumption in~\cref{sec:practical_algorithm}. We define the contained set at iteration~$n$ as~${C_n(\mb{a}, i) = C_{n-1}(\mb{a}, i) \cap Q_n(\mb{a}, i)}$, where~${C_0(\mb{a}, i)}$ is~${[0, \infty]}$ for all~${\mb{a} \in S_0}$ and ~$\mathbb{R}$ otherwise. This ensures that parameters in the initial safe set~$S_0$ remain safe according to the GP model after additional observations. The lower and upper bounds on this set are defined as~${l^i_n(\mb{a}) \coloneqq \min C_n(\mb{a}, i)}$ and~${u^i_n(\mb{a}) \coloneqq \max C_n(\mb{a}, i)}$, respectively. For notational clarity, we write~${l^f_n(\mb{a}) \coloneqq l^0_n(\mb{a})}$ and~${u^f_n(\mb{a}) \coloneqq u^0_n(\mb{a})}$ for the performance bounds.

Based on these confidence intervals for the function values and a current safe set~$S_{n-1}$, we can enlargen the safe set using the Lipschitz continuity properties,
\begin{equation}
S_n = \bigcap_{i \in \mathcal{I}_g} \bigcup_{\mb{a} \in S_{n-1}} \left\{ \mb{a}' \in \mathcal{A} \, | \, l_n^i(\mb{a}) - L \|\mb{a}\hspace{-0.2em} -\hspace{-0.2em} \mb{a}'\| \geq 0 \right\}.
\end{equation}
The set~$S_n$ contains all points in~$S_{n-1}$, as well as all additional parameters that fulfill the safety constraints given the GP confidence intervals and the Lipschitz constant.

With the set of safe parameters defined, the last remaining challenge is to trade off between exploration and exploitation. One could, similar to~\cite{Schreiter2015Safe}, simply select the most uncertain element over the entire set. However, this approach is not sample-efficient, since it involves learning about the entire function rather than restricting evaluations to the relevant parameters. To avoid this, we first define subsets of~$S_n$ that correspond to parameters that could either improve the estimate of the maximum or could expand the safe set. The set of potential maximizers is defined as
\begin{equation}
M_n \coloneqq \left\{ \mb{a} \in S_n \,|\, u^f_n(\mb{a}) \geq \max_{\mb{a}' \in S_n} l^f_n(\mb{a}') \right\},
\label{math:maximizer_set}
\end{equation}
which contains all parameters for which the upper bound of the current performance estimate is above the best lower bound. The parameters in~$M_n$ are candidates for the optimum, since they could obtain performance values above the current conservative estimate of the optimal performance.

Similarly, an optimistic set of parameters that could potentially enlarge the safe set is
\begin{align}
G_n &\coloneqq \left\{ \mb{a} \in S_n \,|\, e_n(\mb{a}) > 0 \right\}, \label{math:expander_set} \\
e_n(\mb{a}) &\coloneqq \big| \big\{ \mb{a}' \in \mathcal{A} \setminus S_n \,|\, \exists  i \in \mathcal{I}_g \colon
\label{math:expander_function} \\
&\phantom{\coloneqq \big| \big\{ \mb{a}' \in \mathcal{A} \setminus S_n \,|\,}
u^i_n(\mb{a}) - L \|\mb{a} - \mb{a}'\| \geq 0 \big\} \big| .  \notag
\end{align}
The function~$e_n$ enumerates the number of parameters that could additionally be classified as safe if a safety function obtained a measurement equal to its upper confidence bound. Thus, the set~$G_n$ is an optimistic set of parameters that could potentially expand the safe set.

We trade off between the two sets,~$M_n$ and~$G_n$, by selecting the most uncertain element across all performance and safety functions; that is, at each iteration~$n$ we select
\begin{align}
\mb{a}_n &= \argmax_{\mb{a} \in G_n \cup M_n} \max_{i \in \mathcal{I}} \, w_n(\mb{a}, i),
\label{math:safeopt_selection_criterion}\\
w_n(\mb{a}, i) &= u^i_n(\mb{a}) - l^i_n(\mb{a}) \label{math:safeopt_uncertainty}
\end{align}
as the next parameter set to be evaluated on the real system. The implications of this selection criterion will become more apparent in the next section, but from a high-level view this criterion leads to a behavior that focuses almost exclusively on exploration initially, as the most uncertain points will typically lie on the boundary of the safe set for many commonly used kernels. This changes once the constraint evaluations return results closer to the safety constraints. At this point, the algorithm keeps switching between selecting parameters that are potential maximizers, and parameters that could expand the safe set and lead to new areas in the parameter space with even higher function values. Pseudocode for the algorithm is found in~\cref{alg:safe_opt}.

We show an example run of the algorithm in~\cref{fig:set_example}. It starts from an initial safe parameter~$\mb{a}_0 \in S_0$ at which we obtain a measurement in~\cref{fig:set_example_initial_safe}. Based on this, the algorithms uses the continuity properties of the safety function and the GP in order to determine nearby parameters as safe (red set). This corresponds to the region where the high-probability confidence intervals of the GP model (blue shaded) are above the safety threshold (grey dashed line). At the next iteration in~\cref{fig:set_example_intermediate_safe}, the algorithm evaluates parameters that are close to the boundary of the safe set, in order to expand the set of safe parameters. Eventually the algorithm converges to the optimal parameters in~\cref{fig:set_example_final}, which obtain the largest performance value that is possible without violating the safety constraints. A local optimization approach, e.g. based on estimated gradients\footnote{If gradient information is available, it can be incorporated in the GP model too~\citep{Solak2003Derivative}}, would have gotten stuck in the local optimum at the initial parameter~$\mb{a}_0$.

At any iteration, we can obtain an estimate for the current best parameters from
\begin{equation}
\hat{\mb{a}}_n = \underset{\mb{a} \in S_n}{\mathrm{argmax}}\, l^f_n(\mb{a}),
\label{math:safeopt_pessimistic_estimate}
\end{equation}
which returns the best, safe lower-bound on the performance function~$f$.

\begin{algorithm}[t]
% \SetAlgoNoEnd
\SetAlgoNoLine
    \caption{\textsc{SafeOpt-MC}}
    \DontPrintSemicolon
    \label{alg:safe_opt}
    \SetKwInOut{Input}{Inputs}
    \Input{Domain $\mathcal{A}$, \newline
           GP prior $k((\mb{a}, i), (\mb{a}', j)),$ \newline
           Lipschitz constant $L$, \newline
           Initial safe set~${S_0 \subseteq \mathcal{A}}$}
    \For{$\ti=1,\dots$}{
        $S_n \gets \underset{i \in \mathcal{I}_g}{\bigcap} \underset{\mb{a} \in S_{n-1}}{\bigcup} \left\{ \mb{a}' \in \mathcal{A} \, | \, l^i_n(\mb{a}) - L \|\mb{a} - \mb{a}'\| \geq 0 \right\}$
        \label{alg:safe_set} \;
        $M_n \gets \left\{ \mb{a} \in S_n \,|\, u^f_n(\mb{a}) \geq \max_{\mb{a}' \in S_n} l^f_n(\mb{a}') \right\}$
        \label{alg:maximizer_set} \;
        $G_n \gets \left\{ \mb{a} \in S_n \,|\, e_n(\mb{a}) \geq 0 \right\}$
        \label{alg:expander_set} \;
        $\mb{a}\ts{\ti} \gets \argmax_{\mb{a} \in G_n \cup M_n} \max_{i \in \mathcal{I}}\, w_n(\mb{a}, i)$
        \label{alg:acquisition} \;
        Measurements:~$\hat{f}(\mb{a}_n),~\hat{g}_i(\mb{a}_n) \,\forall i=0,\dots,q$
        \label{alg:evaluate}\;
        Update GP with new data
        \label{alg:update_gp} \;
        }
\end{algorithm}

\subsection{Theoretical Results}
\label{sec:theoretical_results}

In this section, we show that the same theoretical framework from the~\textsc{SafeOpt} algorithm~\citep{Sui2015Safe} can be extended to multiple constraints and the evaluation criterion~\cref{math:safeopt_selection_criterion}. Here, we only provide the results and high-level ideas of the proofs. The mathematical details are provided in~\cref{sec:proofs}. For simplicity, we assume that all function evaluations are corrupted by the same $\sigma$-sub-Gaussian noise in this section.

In order to provide guarantees for safety, we need the confidence intervals in~\cref{math:gp_confidence_intervals} to hold for all iterations and functions. In the following, we assume that the surrogate function~$h(\mb{a}, i)$ has bounded norm in a reproducing kernel Hilbert space (RKHS, c.f.,~\cite{Christmann2008Support}). A RKHS corresponding to a kernel~$k(\cdot, \cdot)$ includes functions of the form~$h(\mb{a}, i) = \sum_j \alpha_j k((\mb{a}, i), (\mb{a}_j, i_j))$ with~${\alpha_i \in \mathbb{R} }$ and representer points~$(\mb{a}_j, i_j) \in \mathcal{A} \times \mathcal{I}$. The bounded norm property implies that the coefficients~$\alpha_j$ decay sufficiently fast as~$j$ increases. Intuitively, these functions are well-behaved, in that they are regular with respect to the choice of kernel.

The following Lemma allows us to choose a scaling factor~$\beta_n$ for~\cref{math:gp_confidence_intervals}, so that we achieve a specific probability of the true function being contained in the confidence intervals for all iterations.
\begin{restatable}[(based on \citet{Chowdhury2017Kernelized})]{lemma}{confidencethm}
Assume that~${h(\mb{a}, i)}$ has RKHS norm bounded by~$B$ and that measurements are corrupted by~$\sigma$-sub-Gaussian noise. If ${\beta_n^{1/2} = B + 4 \sigma \sqrt{ \gamma_{(n-1)|\mathcal{I}|} + 1 + \mathrm{ln}(1 / \delta)}}$, then the following holds for all parameters~${\mb{a} \in \mathcal{A}}$, function indices~${i \in \mathcal{I}}$, and iterations~${n \geq 1}$ jointly with probability at least~${1 - \delta}$:
\begin{equation}
\big|\, h(\mb{a}, i) - \mu_{n-1}(\mb{a}, i) \,\big| \leq \beta_{n}^{1/2} \sigma_{n-1}(\mb{a}, i).
\end{equation}
\label{thm:confidence_interval}
\end{restatable}
Moreover, if the kernel is continuously differentiable, then the corresponding functions are Lipschitz continuous~\citep{Christmann2008Support}. Note that~\cref{thm:confidence_interval} does not make probabilistic assumptions on $h$ -- in fact, $h$ could be chosen adversarially, as long as it has bounded norm in the RKHS.  Similar results can be be obtained for the Bayesian setting where the function~$h$ is assumed to be drawn from the GP prior~\citep{Srinivas2012Gaussian}.

The scaling factor~$\beta_n$ in \cref{thm:confidence_interval} depends on the {\em information capacity}~$\gamma_n$ associated with the kernel $k$. It is the maximum amount of mutual information that we can obtain about the GP model of~$h(\cdot)$ from $n$ noisy measurements~$\hat{\mb{h}}_\mathcal{D}$ at parameters~${\mathcal{D} = \{(\mb{a}_1, i_1), \dots \}}$,
\begin{equation}
  \gamma_n = \max_{\mathcal{D} \subseteq \mathcal{A} \times \mathcal{I}, |\mathcal{D}| \leq n} \mathrm{I}(\hat{\mb{h}}_\mathcal{D}; h).
  \label{eq:information_capacity}
\end{equation}
Intuitively, it quantifies a best case scenario where we can select the measurements in the most informative manner possible. The information capacity~$\gamma_n$ has a sublinear dependence on~$n$ for many commonly-used kernels and can numerically approximated up to a small constant factor for any given kernel~\cite{Srinivas2012Gaussian}.

Since the confidence intervals hold with probability~${ 1 - \delta}$ and the safe set is not empty starting from~$S_0$, it is possible to prove that parameters within the safe set~$S_n$ are always safe with high probability. In order for the algorithm to compete with our baseline, we must additionally ensure that the algorithm learns the true function up to~$\epsilon$ confidence in both the sets~$M_n$ and~$G_n$. The number of measurements required to achieve this depends on the information capacity~$\gamma_n$, since it encodes how much information can be obtained about the true function from~$n$ measurements. We use the sublinearity of~$\gamma_n$ in order to bound the number of samples required to estimate the function up to~$\epsilon$ accuracy. We have the following result:
\begin{restatable}{theorem}{safeopttheorem}
Assume that~${h(\mb{a}, i)}$ has bounded norm in an RKHS and that the measurement noise is $\sigma$-sub-Gaussian. Also, assume that~${ S_0 \neq \emptyset }$ and~${ g_i(\mb{a}) \geq 0 }$ for all~${ \mb{a} \in S_0 }$ and~${i \in \mathcal{I}_g}$. Choose~$\beta_n$ as in~\cref{thm:confidence_interval}, define~$\hat{\mb{a}}_n$ as in~\cref{math:safeopt_pessimistic_estimate}, and let~$n^*(\epsilon, \delta)$ be the smallest positive integer satisfying
\begin{equation}
\frac{n^*}{\beta_{n^*} \gamma_{|\mathcal{I}| n^*}} \geq \frac{C_1 (|\bar{R}_0(S_0)| + 1)}{\epsilon^2},
\label{eq:thm_safety_n_star}
\end{equation}
where~${C_1 = 8 / \log(1 + \sigma^{-2})}$. For any~${\epsilon > 0}$ and~${\delta \in (0, 1)}$, when running~\cref{alg:safe_opt} the following inequalities jointly hold with probability at least~${1-\delta}$:
\begin{enumerate}
\item {\em Safety:} $\forall n \geq 1, \forall i \in \mathcal{I}_g \colon g_i(\mb{a}_n) \geq 0$
\item {\em Optimality:} $\forall n \geq n^*, \, f(\hat{\mb{a}}_n) \geq f_\epsilon^* - \epsilon$
\end{enumerate}
\label{thm:safeopt_optimality_safety}
\end{restatable}
\begin{proof}
Main idea: safety follows from~\cref{thm:confidence_interval}, since accurate confidence intervals imply that we do not evaluate unsafe parameters. For the optimality, the main idea is that, since we evaluate the most uncertain element in~${M_n \cup G_n}$, the uncertainty about the maximum is bounded by~$w_n(\mb{a}_n, i_n)$. The result follows from showing that, after a finite number of evaluations, either the safe set expands or the maximum uncertainty within~${M_n \cup G_n}$ shrinks to~$\epsilon$. See~\cref{sec:proofs} for derivations and details. \qed
\end{proof}
\cref{thm:safeopt_optimality_safety} states that, given the assumptions we made about the underlying function,~\cref{alg:safe_opt} explores the state space without violating the safety constraints and, after at most~$n^*$ samples, finds an estimate that is $\epsilon$-close to the optimal value over the safely reachable region. The information capacity~$\gamma_{|\mathcal{I}|n^*}$, grows at a faster rate of~$|\mathcal{I}|n$ compared to~$n$ in \textsc{SafeOpt}, since we obtain~$|\mathcal{I}|$ observations at the same parameters set~$\mb{a}$, while the \textsc{SafeOpt} analysis assumes every measurement is optimized independently. However,~${\gamma_{|\mathcal{I}|n}}$ remains sublinear in~$n$, see~\cref{sec:proofs}.

\subsubsection{Contexts}

In this section, we show how the theoretical guarantees derived in the previous section transfer to contextual Bayesian optimization. In this setting,
part of the variables that influence the performance, the contexts, are fixed by an external process that we do not necessarily control. In normal Bayesian optimization, it was shown by~\cite{Krause2011Contextual} that an algorithm that optimizes the GP-UCB criterion in~\cref{eq:contextual_gp_ucb} for a fixed context converges to the global optimum after repeatedly seeing the corresponding context.

Intuitively, the fact that part of the variables that influence the performance, the contexts, are now fixed by an external process should not impact the algorithm on a fundamental level. However, safety is a critical issue in our experiments and, in general, one could always select an adversarial context for which we do not have sufficient knowledge to determine safe controller parameters. As a consequence, we make the additional assumption that only `safe' contexts are visited; that is, we assume the following:
\begin{assumption}
For any~$n\geq1$, the context~$\mb{z}_n \in \mathcal{Z}$ is selected such that~$S_n(\mb{z}_n) \neq \emptyset$.
\label{as:context_always_safe}
\end{assumption}
Here,~$S_n(\mb{z}_n)$ denotes the safe set for the given context~$\mb{z}_n$.
Intuitively, \cref{as:context_always_safe} ensures that for every context there exists at least one parameter choice that is known to satisfy all safety constraints. This assumption includes the case where safe initial parameters  for all contexts are known~\textit{a priori} and the case where the algorithm terminates and asks for help from a domain-expert whenever a context leads to an empty safe set.

\begin{figure*}
  \includegraphics[scale=1]{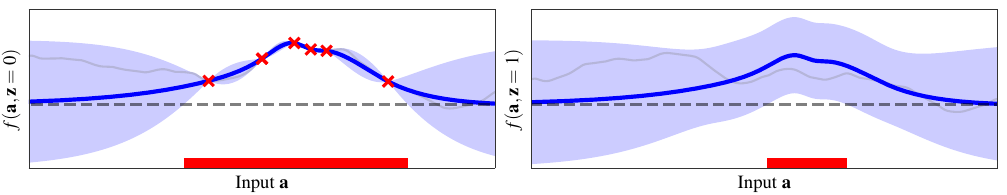}
  \caption{Example run of the context-dependent \textsc{SafeOpt-MC} algorithm. For the first six samples, the algorithm only sees the context~$\mb{z}=0$ (left function) and obtains measurements there (red crosses). However, by exploiting correlations between different contexts, the algorithm can transfer knowledge about the shape of the function and safe set over to a different context,~$\mb{z}=1$ (right function). This enables the algorithm to be significantly more data-efficient.  \label{fig:context_transfer}}
\end{figure*}

A trivial extension of \textsc{SafeOpt-MC} to contexts is to run~$|\mathcal{Z}|$ independent instances of~\cref{alg:safe_opt}, one for each context. This way, it is sufficient to repeatedly see a context several times to apply the previous results to the safe contextual optimization case. One can apply the previous analysis to this setting, but it would fail to yield guarantes that hold {\em jointly} for all contexts.

In order to obtain stronger results that hold jointly across all contexts in~$\mathcal{Z}$, we adapt the information capacity (worst-case mutual information)~$\gamma_n$ to consider contexts,
\begin{equation}
  \gamma_n = \max_{\mathcal{D} \subseteq \mathcal{A} \times \mathcal{Z} \times \mathcal{I}, |\mathcal{D}| \leq n} \mathrm{I}(\hat{\mb{h}}_\mathcal{D}; h),
  \label{eq:information_capacity_with_context}.
\end{equation}
Unlike in~\cref{eq:information_capacity}, the mutual information is maximized across contexts in~\cref{eq:information_capacity_with_context}. As a result, we can use~\cref{thm:confidence_interval} to obtain confidence intervals that hold jointly across all contexts.

A second challenge is that contexts are chosen in an arbitrary order. This is in stark contrast to the parameters~$\mb{a}_n$, which are chosen according to~\cref{math:safeopt_selection_criterion} in order to be informative. This means that any tight finite sample bound on~\cref{alg:safe_opt} must necessarily depend on the order of contexts. The following theorem accounts for both of these challenges.

\begin{theorem}
  Under the assumptions of~\cref{thm:safeopt_optimality_safety} and~\cref{as:context_always_safe}. Choose~$\beta_n$ as in~\cref{thm:confidence_interval}, where~$\gamma_n$ is now the worst-case mutual information over contexts as in~\cref{eq:information_capacity_with_context}. For a given context order~$(\mb{z}_1, \mb{z}_2,\dots)$ and any context~$\mb{z} \in \mathcal{Z}$, let
  \begin{equation}
    n(\mb{z}) = \sum_{n=1}^{n^*(\mb{z})} \mathds{1}_{\mb{z} = \mb{z}_n}
  \end{equation}
  be the number of times that we have observed the context~$\mb{z}$ up to iteration~$n^*$ and $\mathds{1}$ is the indicator function. Let $n^*(\mb{z})$ be the smallest positive integers such that
\begin{equation}
\frac{n(\mb{z})}{\beta_{n^*(\mb{z})}~ \gamma_{ n(\mb{z}) | \mathcal{I}|}(\mb{z})} \geq \frac{C_1 (|\bar{R}_0(S_0(\mb{z}))| + 1)}{\epsilon^2},
\end{equation}
where~${C_1 = 8 / \log(1 + \sigma^{-2})}$.
We note the information capacity for a fixed context~$\mb{z}$ by ~$\gamma_n(\mb{z})$. That is, with ${h_\mb{z}(\mb{a}, i) = h(\mb{a}, i, \mb{z})}$ it is defined as
$\gamma_n(\mb{z}) = \max_{\mathcal{D} \subseteq \mathcal{A} \times \mathcal{I}, |\mathcal{D}| \leq n} \mathrm{I}(\hat{\mb{h}_\mb{z}\,}_\mathcal{D}; h_\mb{z})$.  For any~${\epsilon > 0}$ and~${\delta \in (0, 1)}$, let $f_\epsilon^*(\mb{z}) = \max_{\mb{a} \in \bar{R}_\epsilon(S_0)} f(\mb{a}, \mb{z})$. Then, when running~\cref{alg:safe_opt} the following inequalities jointly hold with probability at least~${1-\delta}$:
\begin{enumerate}
\item $\forall n \geq 1, i \in \mathcal{I}_g \colon g_i(\mb{a}_n, \mb{z}_n) \geq 0$
\item $\forall \mb{z} \in \mathcal{Z}, n \geq n^*(\mb{z}): \, f(\hat{\mb{a}}_n, \mb{z}) \geq f_\epsilon^*(\mb{z}) - \epsilon$
\end{enumerate}
\label{thm:safeopt_optimality_safety_context}
\end{theorem}
\begin{proof}
  For a fixed context, ${\mb{z}_n = \mb{z}\, \forall n}$, we have~$n^*(\mb{z}) = n(\mb{z})$ and the results follow directly as in~\cref{thm:safeopt_optimality_safety}. Otherwise, the only difference in the proofs is that~$\beta$ depends on the information capacity over contexts, making sure that the confidence intervals are valid across contexts. By visiting contexts in~${\mathcal{Z} \setminus \{ \mb{z} \} }$, we obtain more measurements and~$\beta$ increases, which in turn increases the upper bound on the number of samples required at context~$\mb{z}$.
  \qed
\end{proof}

\cref{thm:safeopt_optimality_safety_context} states that the contextual variant of~\cref{alg:safe_opt} enjoys the same safety guarantees as the non-contextual version. Additionally, it shows that, after gaining enough information about a particular context, it can identify the optimal parameters. In practice, this upper bound is conservative, since it does not acount for knowledge transfer accross contexts. In practice, correlations between contexts significantly speed up the learning process. For example, in~Figure~\ref{fig:context_transfer} we show a contextual safe optimization problem with two contexts. Even though the algorithm has only been able to explore the parameter space at the first context ($\mb{z}=0$, left function), the correlation between the functions means that information can be transferred to the as-of-yet unobserved context (${\mb{z}=1}$, right function). This knowledge transfer significantly improves data-efficiency and the number of evaluations required by the algorithm.

%%%%%%%%%%%%%%%%%%%%%%%%%%%%%%%%%%%%%%%%%%%%%%%%%%%%%%%%%%%%%%%%%%%%
%%%%%%%%%%%%%%%%%%%%%%%%%%%%%%%%%%%%%%%%%%%%%%%%%%%%%%%%%%%%%%%%%%%%

\subsection{Practical Implementation}
\label{sec:practical_algorithm}

In this section, we discuss possible changes to~\cref{alg:safe_opt} that make the algorithm more practical, at the expense of loosing some of the theoretical guarantees. One challenge in applying~\cref{alg:safe_opt} in practice, is defining a suitable Lipschitz constant. In particular, specifying the wrong constant can lead to conservativeness or unsafe parameters being evaluated. Moreover, smoothness assumptions are already encoded by the kernel choice, which is more intuitive to specify than Lipschitz constants on their own. In practice, we use only the GP model to ensure safety~\citep{Berkenkamp2016Safe}; that is, we define~$l^i_n(\mb{a}) = \min Q_n(\mb{a}, i)$ and~$u^i_n(\mb{a}, i) = \max Q_n(\mb{a}, i)$ in terms of the confidence intervals of the GP directly. Thus, we can define the safe set without a Lipschitz constant as
\begin{equation}
S_n = S_0 \cup \left\{ \mb{a} \in \mathcal{A} ~|~\forall i \in \mathcal{I}_g \colon l^i_n(\mb{a}) \geq 0 \right\}.
\label{math:confidence_safe_set}
\end{equation}
While it is difficult to prove the full exploration of the safely reachable set as in~\cref{thm:safeopt_optimality_safety}, the resulting algorithm remains safe:
\begin{lemma}
With the assumptions of~\cref{thm:confidence_interval}, ${ S_0 \neq \emptyset }$, and~${g_i(\mb{a}) \geq 0}$ for all~${ \mb{a} \in S_0 }$ and~${ i \in \mathcal{I}_g}$, when running \cref{alg:safe_opt} with the safe set defined as in~\cref{math:confidence_safe_set}, the following holds with probability at least~${1 - \delta}$:
\begin{equation}
\forall n \geq 1, \, \forall i \in \mathcal{I}_g \colon\, g_i(\mb{a}_n) \geq 0.
\end{equation}
\label{thm:safety_only}
\end{lemma}
\begin{proof}
The confidence intervals hold with probability~$1-\delta$ following~\cref{thm:confidence_interval}. Since~$S_n$ in~\cref{math:confidence_safe_set} is defined as the set of parameters that fulfill the safety constraint and the safe set is never empty since~${S_0 \neq \emptyset}$, the claim follows. \qed
\end{proof}

Similarly, the set of expanders can be defined in terms of the GP directly, by adding optimistic measurements and counting the number of new parameters that are classified as safe, see~\citep{Berkenkamp2016Safe} for details. However, this potentially adds a large computational burden.

The parameter~$\beta_n$, which determines the GP's confidence interval in~\cref{thm:confidence_interval}, may be impractically conservative for experiments. The theoretical safety results also hold when we replace~$\gamma_n$ in~$\beta_n$ by the empirical mutual information gained so far,~$\mathrm{I}(\hat{\mb{h}}_{\mathcal{D}_n \times \mathcal{I}}, h)$. Empirically, depending on the application, one may also consider setting~$\beta_n$ to a constant value. This roughly corresponds to bounding the failure probability per iteration, rather than over all iterations.

Learning all the different functions,~$f$ and~$g_i$, up to the same accuracy~$\epsilon$ may be restrictive if they are scaled differently. A possible solution is to either scale the observed data, or to scale the uncertainties in~\cref{math:safeopt_selection_criterion} by the prior variances of the kernels for that specific output. This enables~\cref{math:safeopt_selection_criterion} to make more homogeneous decisions across different scales.

%!TEX root = ../root.tex

\section{Quadrotor Experiments}
\label{sec:results}

In this section, we demonstrate~\cref{alg:safe_opt} (with the changes discussed in~\cref{sec:practical_algorithm}) in experiments on a quadrotor vehicle, a Parrot AR.Drone 2.0.

A Python implementation of the \textsc{SafeOpt-MC} algorithm that builds on~\citepalias{TheGPyauthors2012GPy}, a GP library, is available at \mbox{\url{http://github.com/befelix/SafeOpt}}. Videos of the experiments can be found at
\begin{itemize}
	\item \cref{sec:exp_linear_control}: \mbox{\url{http://tiny.cc/icra16_video}}
	\item \cref{sec:experiment_step_response}: \mbox{\url{https://youtu.be/rLmwYtoE3yg}}
	\item \cref{sec:experiment_circle}: \mbox{\url{https://youtu.be/4xC4OSiIDGk}}
\end{itemize}

\subsection{Experimental Setup}

During the experiments, measurements of all vehicle states are estimated from position and pose data provided by an overhead motion capture camera system. The quadrotor's dynamics can be described by six states: positions~${\mb{x} = (x, y, z)}$, velocities~${\dot{\mb{x}} = (\dot{x}, \dot{y}, \dot{z})}$, ZYX Euler angles~${(\phi, \theta, \psi)}$, and body angular velocities~${(\omega_x, \omega_y, \omega_z)}$. The control inputs~$\mb{u}$ are the desired roll and pitch angles~$\theta_{\mathrm{des}}$ and~$\phi_{\mathrm{des}}$, the desired~$z$-velocity~$\dot{z}_{\mathrm{des}}$, and the desired yaw angular velocity~$\omega_{z,\mathrm{des}}$, which in turn are inputs to an unknown, proprietary, on-board controller.

The position dynamics in the global coordinate frame are
\begin{equation}
\ddot{\mb{x}} = R_{\mathrm{ZYX}}(\phi, \theta, \psi) \vec{f} - \vec{g},
\label{math:quadrotor_dynamics}
\end{equation}
where~$R_{\mathrm{ZYX}}$ is the rotation matrix from the body frame to the inertial frame,~${\vec{f}=(0,0,c)}$ is the mass-normalized thrust, and~${\vec{g}=(0,0,g)}$ is the gravitational force. The goal of the controller is to track a reference signal. We assume that $z$-position and the yaw angle are controlled by fixed control laws and focus on the position control in $x$- and $y$ direction. We use two different control laws in the following experiments.

The most simple control law that can be used for this setting is a PD-controller, defined as
\begin{align}
\phi_\mathrm{des} &= k_1 (x_k - x_\mathrm{des}) + k_2 (\dot{x} - \dot{x}_\mathrm{des}), \label{eq:linear_ctrl_x} \\
\theta_\mathrm{des} &= k_1 (y_k - y_\mathrm{des}) + k_2 (\dot{y} - \dot{y}_\mathrm{des}),
\end{align}
where~${\mb{a}=(k_1, k_2)}$ are the two tuning parameters. Intuitively, a larger parameter~$k_1$ encourages tracking reference changes more aggressively, while~$k_2$ is a damping factor that encourages moderate velocities.

A more sophisticated approach to control quadrotor vehicles is to use estimates of the angles and accelerations to solve for the thrust~$c$.
We use loop shaping on the horizontal position control loops so that they behave in the manner of a second-order systems with time constant~$\tau$ and damping ratio~$\zeta$. Based on a given desired reference trajectory, commanded accelerations are computed as
\begin{align}
\ddot{x}_c &= \frac{1}{\tau^2} (x_{\mathrm{des}} - x) + \frac{2 \zeta}{\tau} (\dot{x}_{\mathrm{des}} - \dot{x}), \label{math:xdd_desired} \\
\ddot{y}_c &= \frac{1}{\tau^2} (y_{\mathrm{des}} - y) + \frac{2 \zeta}{\tau} (\dot{y}_{\mathrm{des}} - \dot{y}). \label{math:ydd_desired}
\end{align}
From the commanded accelerations, we then obtain the control inputs for the desired roll and pitch angles by solving~\cref{math:quadrotor_dynamics} for the angles. Here, the tuning parameters are~${\mb{a}=(\tau, \zeta)}$. For details regarding the controllers see~\citep{Schoellig2012Feedforward,Lupashin2014Platform}.

The quadrotor was controlled using the ardrone\_autonomy and vicon\_bridge packages in ROS Hydro. Computations for the \textsc{SafeOpt-MC} algorithm in~\cref{alg:safe_opt} were conducted on a regular laptop and took significantly less than one second per iteration. As a result, experiments could be conducted continuously without interruptions or human interventions.

\subsection{Kernel Selection}

\cref{alg:safe_opt} critically depends on the GP model for the performance and constraint functions. In this section, we review the kernel used in our experiments and the kind of models that they encode. A more thorough review of kernel properties can be found in~\citep{Rasmussen2006Gaussian}.

In our experiments, we use the Mat{\`e}rn kernel with parameter~${\nu = 3/2}$~\citep{Rasmussen2006Gaussian},
\begin{align}
k(\mb{a}, \mb{a}') &= \kappa^2 \left( \hspace{-0.25em} 1 \hspace{-0.15em} + \hspace{-0.2em} \sqrt{3}\,r(\mb{a}, \mb{a}') \hspace{-0.25em} \right) \exp \left(\hspace{-0.25em}- \hspace{-0.05em} \sqrt{3} \,r(\mb{a}, \mb{a}') \hspace{-0.25em} \right) \hspace{-0.25em},
\label{math:matern_kernel} \\
r(\mb{a}, \mb{a}') &=  \sqrt{ (\mb{a} - \mb{a}')^\T \mb{M}^{-2} (\mb{a} - \mb{a}') },
\end{align}
which encodes that mean functions are one-times differentiable. This is in contrast to the commonly used squared exponential kernels, which encode smooth (infinitely differentiable) functions. With the Mat{\`e}rn kernel, the GP model is parameterized by three hyperparameters: measurement noise~$\sigma^2$ in~\cref{math:gp_prediction_mean,math:gp_prediction_variance}, the kernel's prior variance~$\kappa^2$, and positive lengthscales~${\mb{l} \in \mathcal{R}_+^{\mathcal{A}}}$, which are the diagonal elements of the diagonal matrix~$\mb{M}$, ${\mb{M} = \mathrm{diag}(\mb{l})}$. These hyperparameters have intuitive interpretations: the variance of the measurement noise~$\kappa^2$ corresponds to the noise in the observations, which includes any randomness in the algorithm and initial conditions, and random disturbances. The prior variance~$\kappa^2$ determines the expected magnitude of function values; that is,~${|f(\mb{a})| \leq \kappa}$ with probability~$0.68$ according to the GP prior. Lastly, the lengthscales~$\mb{l}$ determine how quickly the covariance between neighboring values deteriorates with their distance. The smaller the lengthscales, the faster the function values can change from one parameter set to the next. In particular, the high-probability Lipschitz constant encoded by this kernel depends on the ratio between the prior variance and the lengthscales,~${\kappa / \mb{l}}$.

When using GPs to model dynamic systems, typically a maximum likelihood estimate of the hyperparameters is used based on data; see~\citep{Ostafew2016Robust} for an example. For Bayesian optimization, the GP model is used to actively acquire data, rather than only using it for regression based on existing data. This dependence between the kernel hyperparameters and the acquired data is known to lead to poor results in Bayesian optimization when using a maximum likelihood estimate of the hyperparameters during data acquisition~\citep{Bull2011Convergence}. In particular, the corresponding GP estimate is not guaranteed to contain the true function as in~\cref{thm:confidence_interval}. In this work, we critically rely on these confidence bounds to guarantee safety. As a consequence, we do not adapt the hyperparameters as more data becomes available, but treat the kernel as a prior over functions in the true Bayesian sense; that is, the kernel hyperparameters encode our prior knowledge about the functions that we model and are fixed before experiments begin. While this requires intuition about the process, intuitively the less knowledge we encode in the prior, the more data and evaluations on the real system are required in order to determine the best parameters.

\subsection{Linear Control}
\label{sec:exp_linear_control}

In this section, we use \textsc{SafeOpt-MC} to optimize the parameters of the linear control law in~\cref{eq:linear_ctrl_x}. The entire control algorithm consists of this control law together with the on-board controller and state estimation.

The goal is to find controller parameters that maximize the performance during a 1-meter reference position change. For an experiment with parameters~$\mb{a}_\ti$ at iteration~$\ti$, the performance function is defined as
\begin{align}
f(\mb{a}_\ti) &= c(\mb{a}_\ti) - 0.95\, c(\mb{a}_0),
\label{math:icra_performance_function} \\
c( \mb{a}_\ti ) &= -\sum_{k=0}^N \mathbf{x}_k^\mathrm{T} \mathbf{Q} \mathbf{x}_k + R u_k^2,
\label{math:icra_cost_function}
\end{align}
where, to compute the cost~$c$, the states~${\mb{x} = (x - 1, \dot{x}, \phi, \omega)}$ and the input~$u$ are weighted by positive semi-definite matrices~$\mathbf{Q}$ and~$R$. The subscript~$k$ indicates the state measurement at time step~$k$ in the trajectory and the time horizon is~\unit[5]{s}~(${N=350}$).
Performance is defined as the cost improvement relative to $95\%$ of the initial controller cost. The safety constraint is defined only in terms of the performance; that is, the constraint is~$g(\mb{a})=f(\mb{a})\geq 0$, which encodes that we do not want to evaluate controller parameters that perform significantly worse than the initial parameters.

While the optimal controller parameters could be easily computed given an accurate model of the system, we do not have a model of the dynamics of the proprietary, on-board controller and the time delays in the system.
Moreover, we want to optimize the performance for the real, nonlinear quadrotor system, which is difficult to model accurately.
An inaccurate model of the system could be used to improve the prior GP model of the performance function, with the goal of achieving faster convergence. In this case, the uncertainty in the GP model of the performance function would account for inaccuracies in the system model.

We define the domain of the controller parameters as~$[-0.6, 0.1]^2$, explicitly including positive controller parameters that certainly lead to crashes.
In practice, one would exclude parameters that are known to be unsafe~\textit{a priori}.
The initial controller parameters are $(-0.4, -0.4)$, which result in a controller with poor performance.
Decreasing the controller gains further leads to unstable controllers.

\begin{figure}[t]
\includegraphics[scale=1]{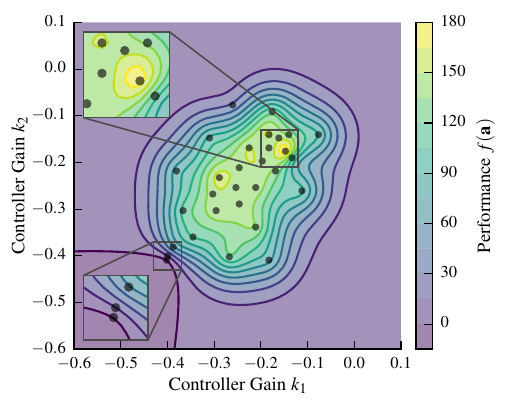}
\caption{GP mean estimate of the performance function after~$30$ evaluations. The algorithm adaptively decides which parameters to evaluate based on safety and informativeness. In the bottom-left corner, there is the magnified section of the first three samples, which are close together to determine the location of the initial, safe region. The maximum, magnified in the top-left corner, also has more samples to determine the precise location of the maximum. Other areas are more coarsely sampled to expand the safe region.}
\label{fig:icra_performance}
\end{figure}

The parameters for the experiments were set as follows: the length-scales were set to~$0.05$ for both parameters, which corresponds to the notion that a~0.05-0.1 change in the parameters leads to very different performance values. The prior standard deviation,~$\kappa$, and the noise standard deviation,~$\sigma$, are set to~$5\%$ and~$10\%$ of the performance of the inital controller,~$f(\mb{a}_0)$, respectively. The noise standard deviation,~$\sigma$, mostly models errors due to initial position offsets, since state measurements have low noise. The size of these errors depends on the choice of the matrices~$\mb{Q}$ and~$R$. By choosing~$\sigma$ as a function of the initial performance, we account for the~$\mb{Q}$ and~$R$ dependency. Similarly,~$\kappa$ specifies the expected size of the performance function values. Initially, the best we can do is to set this quantity dependent on the initial performance and leave additional room for future, larger performance values. For the GP model, we choose~${\beta^{1/2}_n=2}$ to define the confidence interval in~\cref{math:gp_confidence_intervals}.

The resulting, estimated performance function after running~\cref{alg:safe_opt} for~30~experiments is shown in Fig.~\ref{fig:icra_performance}.
The unknown function has been reliably identified. Samples are spread out over the entire safe set, with more samples close to the maximum of the function and close to the initial controller parameters. No unsafe parameters below the safety threshold were evaluated on the real system.

Typically, the optimization behavior of~\cref{alg:safe_opt} can be roughly separated into three stages.
Initially, the algorithm evaluates controller parameters close to the initial parameters in order for the GP to acquire information about the safe set (see lower-left, zoomed-in section in~\cref{fig:icra_performance}).
Once a region of safe controller parameters is determined, the algorithm evaluates the performance function more coarsely in order to expand the safe set.
Eventually, the controller parameters are refined by evaluating high-performance parameters that are potential maximizers in a finer grid (see upper-left, zoomed-in section in~\cref{fig:icra_performance}).
The trajectories of the initial, best and intermediate controllers can be seen in~\cref{fig:icra_trajectories}.

\begin{figure}[t]
\includegraphics[scale=1]{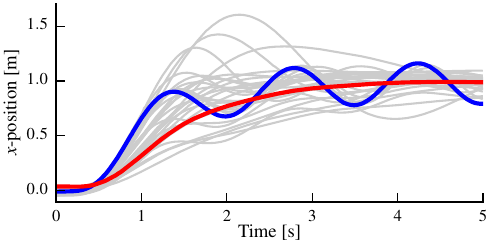}
\caption{The quadrotor controller performance is evaluated during a~\unit[5]{s} evaluation interval, where a~\unit[1]{m} reference position change must be performed. The trajectories correspond to the optimization routine in~\cref{fig:icra_performance}.
The initial controller (blue) performs poorly but is stable.
In contrast, the optimized controller (red) shows an optimized, smooth, and fast response.
The trajectories of other controller parameters that were evaluated are shown in gray.}
\label{fig:icra_trajectories}
\end{figure}

\subsection{Nonlinear Control}
\label{sec:experiment_step_response}

\begin{figure}[t]
\includegraphics[scale=1]{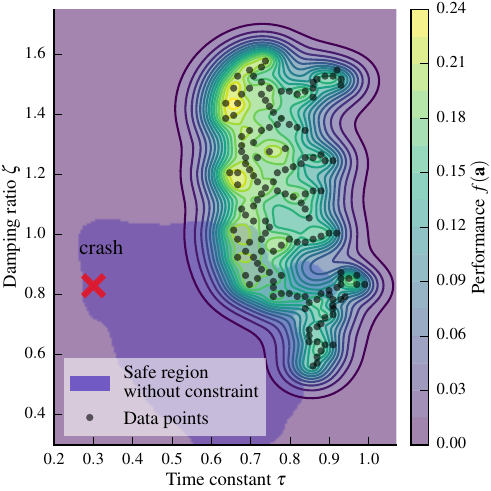}
\caption{Mean estimate of the root-mean-square error when optimizing the parameters of the nonlinear control law for a step response, subject to safety constraints. The algorithm carefully evaluates only safe parameter combinations, until the safe region cannot be expanded further without violating constraints. Without the safety constraint, the algorithm explores a larger region of the parameter space (light blue) and eventually evaluates an unsafe parameter set.}
\label{fig:step_input}
\end{figure}

In the previous section, we showed how to optimize the performance of a linear control law subject to a simple constraint on performance. In this section, we optimize the nonlinear controller in~\cref{math:xdd_desired,math:ydd_desired} and show how more complex constraints can be used.

We use the same task as in the previous section, but this time the goal is to minimize the root-mean-square error (RMSE) over a time horizon of~\unit[5]{s}~(${N =350}$~samples) during a 1-meter reference position change in $x$-direction. We define the performance function,
\begin{align}
f(\mb{a}_n) &= c(\mb{a}_n) - 0.75\, c(\mb{a}_0),
\label{math:performance_function} \\
c ( \mb{a}_n ) &= \frac{1}{\sqrt{N}} \left( \sum_{k=1}^N \|\mb{x}_k - \mb{x}_{\mathrm{des},k} \|_2^2 \right)^{1/2},
\label{math:cost_function}
\end{align}
as the performance relative to~$75\%$ of the performance of the initial parameters~${\mb{a}_0 = (0.9, 0.8)}$. We define the GP model of this performance function as follows: in this experiment, measurement noise is minimal, since the positions are measured accurately by the overhead camera system. However, to capture errors in the initial position, we define~${ \sigma = 0.05 c(\mb{a}_0) }$. We assume that we can improve the initial controller by roughly~$20\%$, so we set~${\kappa = 0.2 c(\mb{a}_0)}$. The lengthscales are set to~$0.05$ in order to encourage cautious exploration. These parameters turned out to be conservative for the real system. Notice that the cost is specified relative to~$c(\mb{a}_0)$ instead of~$f(\mb{a}_0)$ as in~\cref{sec:exp_linear_control}. Since $c(\mb{a}_0) > f(\mb{a}_0)$, these hyperparameters are more conservative, so that we require more evaluations on the real system. The reason for this more conservative choice is that the nonlinear controller is expected to have a less smooth performance function, unlike the one for linear control, which is expected to be roughly quadratic.

If, as in the previous section, one were to set the safety constraint to~${g_1(\mb{a}) = f(\mb{a})}$, the algorithm would classify the blue shaded region in~\cref{fig:step_input} as safe. This region includes time constants as low as~${\tau=0.3}$, which encourage highly aggressive maneuvers, as would be expected from a performance function that encourages changing position as fast as possible. However, these high gains amplify noise in the measurements, which can lead to crashes; that is, the performance-based constraint cannot properly encode safety. Notice that the blue shaded area does not correspond to full exploration, since the experiment was aborted after the first, serious crash. The reason for the unsafe exploration is that the RMSE performance function in~\cref{math:cost_function} does not encode safety the same way as as weighting of state and input errors in~\cref{fig:icra_performance} does. Thus, in order to encode safety, we need to specify additional safety constraints.

One indication of unsafe behavior in quadrotors are high angular velocities when the quadrotor oscillates around the reference point. We define an additional safety constraint on the maximum angular velocity~${\max_{k} |\omega_{x,k}| \leq \unit[0.5]{rad/s}}$ by setting ${g_2(\mb{a}) = 0.5 - \max_{k} |\omega_{x,k}| }$. The corresponding hyperparameters are selected as~${\sigma_2 = 0.1}$, ${l=0.2}$, and~${\kappa=0.25}$. The measurement noise can be chosen relatively small here, since it corresponds to a single measurement of angular velocity. Note that it is difficult to perform the step maneuver with an angular velocity lower than~\unit[0.4]{rad/s}, so that typical values of~$g_2$ are smaller than~$0.1$.

With this additional safety constraint, the algorithm explores the parameter space and stops before the safety constraints are violated, as can be seen in~\cref{fig:step_input}. Rather than exploring smaller time constants~$\tau$ (higher gains), the algorithm evaluates larger damping ratios, which allow slightly smaller values of~$\tau$ and therefore higher performance without violating the safety constraints. The optimal parameters are to the top-left of the safe set, where small time constants encourage tracking the reference aggressively, while the increased damping ratio ensures a moderate angular velocity.

\subsection{Circle Trajectory}
\label{sec:experiment_circle}

\begin{figure}
\includegraphics[scale=1]{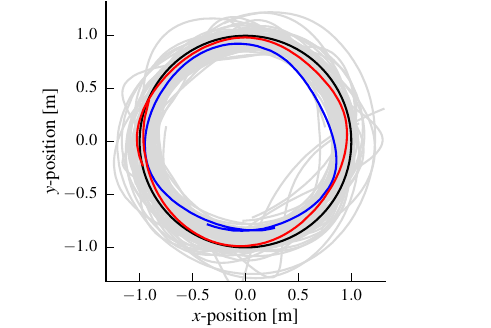}
\caption{The trajectories (gray) resulting from iteratively optimizing the controller parameters for a unit circle reference trajectory at~\unit[1]{m/s} (black). The trajectory with the initial parameters (blue) has poor tracking performance, while the optimized parameters (red) perform significantly better. The flight is safe, i.e., only safe parameters are evaluated.}
\label{fig:circle_trajectories}
\end{figure}

In this section, we use the same nonlinear controller and cost function as in the previous section, but aim to optimize the RMSE with respect to a circle trajectory of radius~\unit[1]{m} at a speed of~\unit[1]{m/s}. The reference is defined as a point moving along the circle at the desired speed. Feasibility of such motions has been analyzed in~\cite{Schoellig2011Feasiblity}.

\begin{figure*}[t]
\centering
\subfloat[Speed: $\dot{x}_\mathrm{des} = 1\,\frac{\mathrm{m}}{\mathrm{s}}$]{
	\includegraphics[scale=1]{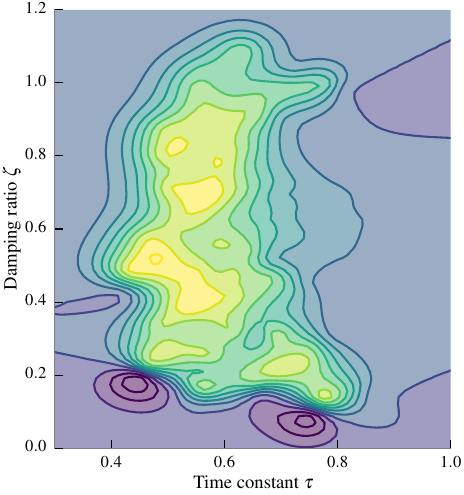}
	\label{fig:context_circle_slow}}
\hfill
\subfloat[Speed: $\dot{x}_\mathrm{des} = 1.8\,\frac{\mathrm{m}}{\mathrm{s}}$]{
	\includegraphics[scale=1]{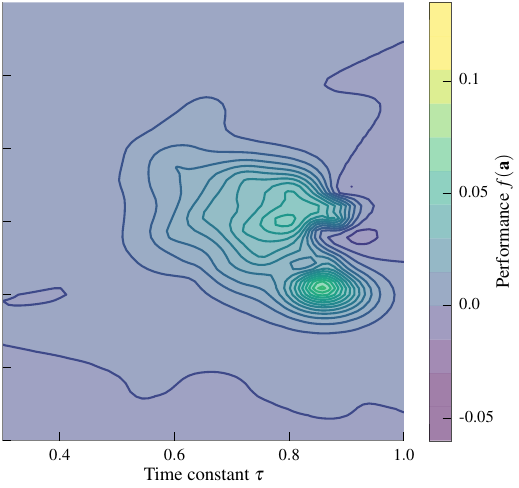}
	\label{fig:context_circle_fast}}
\caption{The mean estimate of the performance function for the circle trajectory in~\cref{fig:circle_trajectories} for a speed of~\unit[1]{m/s} (left) and~\unit[1.8]{m/s} (right). Extending the kernel with a context for speed allows to transfer knowledge to different speeds and leads to speed-dependent optimal control parameters, speeding up the learning for higher speeds.}
\label{fig:context_circle}
\end{figure*}

In order to ensure good tracking behavior, we define safety as a constraint on the maximum RMSE of \unit[0.2]{m}. Additionally, we use the same constraint on the maximum angular velocity around the~$x$ and~$y$ axis of~\unit[0.5]{rad/s} as before.
The yaw-angle is set so that the quadrotor always points to the center of the circle, which ideally should lead to zero angular velocity. Deviations from this are an indication of unsafe behavior. We use the same kernel hyperparameters as in~\cref{sec:experiment_step_response}.

The trajectories that result from running the optimization algorithm are shown in~\cref{fig:circle_trajectories}. The initial controller parameters lead to very poor performance. In particular, the initial time constant is too large, so that the quadrotor lags behind the reference. As a result, the quadrotor flies a circle of smaller radius. In contrast, the resulting optimized trajectory (in red) is the best that can be obtained given the safety constraints and controller structure above. The mean estimate of the performance function after the experiments can be seen in~\cref{fig:context_circle_slow}. The optimal parameters have smaller time constants, so that the position is tracked aggressively. Since the reference point moves of~\unit[1]{m/s}, these aggressive controller parameters do not lead to unsafe behavior. During the entire optimization, only safe parameters that keep the vehicle within the constraints on RMSE and angular velocity are evaluated.

\subsection{Context-Dependent Optimization}
\label{sec:experiment_context}

In this section, we show how the knowledge about good controller parameters at low speeds can be used to speed up the safe learning at higher speeds.

In our circle experiment, the quadrotor tracked a moving reference. As this reference moves with high velocities, the quadrotor gets pushed to its physical actuator limits and starts to lag behind the reference. This causes the circle that is flown by the quadrotor to have a smaller radius than the reference trajectory. In this section, the goal is to maximize the speed of the reference trajectory subject to the safety constraints of the previous experiment in~\cref{sec:experiment_circle}. One way to achieve this goal, is to add the speed of the reference point to the performance function. However, this would lead to more experiments, as the algorithm will explore the safe parameter space for every velocity. Instead, here we define the trajectory speed as a context, which is set externally. In particular, we set
\begin{equation}
z_n = \argmax_{v \in \mathbb{R},\, \mb{a} \in \mathcal{A}} v \quad \textnormal{subject to: }\,\, g_i(\mb{a}, v) \geq 0,\, \forall i \in \mathcal{I}_g,
\label{eq:speed_context}
\end{equation}
that is, we select the maximum velocity for which there are safe parameters known. While here we select the context manually, in practice contexts can be used to model any external, measurable variables, such as the battery level, see~\cref{sec:contextual_bayesian_optimization}.

In order to transfer knowledge about good controller parameters from the slow speed in~\cref{sec:experiment_circle} to higher speeds, we model how performance and constraints vary with desired speed by defining a kernel~$k_z(\dot{x}_\mathrm{des}, \dot{x}_\mathrm{des}')$ over contexts. We use the same kernel structure as in~\cref{eq:context_kernel} and hyperparameters~${\kappa = 1}$ and~$l=0.25$. Based on the data from~\cref{sec:experiment_circle}, the extended model allows us to determine speeds for which safe controller parameters are known.

Starting from the data of the previous experiments in~\cref{sec:experiment_circle}, we run~\textsc{SafeOpt-MC} using the extended kernel with the additional speed context determined by~\cref{eq:speed_context}. This allows us to find optimal parameters for increasingly higher speeds, which satisfy the constraints. We can safely increase the speed up to~\unit[1.8]{m/s}. We show the mean performance function estimates for two speeds in ~\cref{fig:context_circle}. For lower speeds, the best controller parameters track the reference position more aggressively (low~$\tau$). For higher speeds, this behavior becomes unsafe as the quadrotor lags behind the reference point. Instead, the optimal parameters shift to higher time constants (lower gains). Additionally, as expected, high speeds lead to higher reference tracking errors. Increasing the reference velocity any further causes the performance constraint to be violated.

The trajectories that result from applying the optimal parameters for a speed of~\unit[1]{m/s} and the maximum safe speed of~\unit[1.8]{m/s} can be seen in~\cref{fig:circle_trajectories_context}. For the relatively slow speed of~\unit[1]{m/s} the quadrotor can track the circle well using aggressive parameters. For the higher speed, the reference trajectory moves too fast for the quadrotor to track perfectly within the actuator limits, so that the best parameters just barely satisfy the safety constraint on the average deviation from the reference. Overall, this approach allows us to find context-dependent parameters, while remaining within the safety constraints.

\begin{figure}[t]
\includegraphics[scale=1]{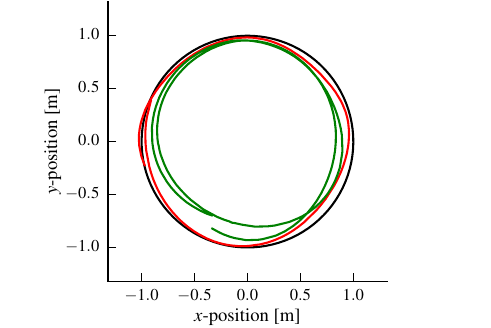}
\caption{Trajectories with optimal parameters for speeds of~\unit[1]{m/s} (red) and~\unit{1.8}[m/s] (green) when tracking the black reference. At slower speeds there exist aggressive controller parameters that allow the quadrotor to track the reference almost perfectly. At higher speeds, actuator saturation limits the achievable performance. Due to the safe optimization framework the maximum speed can be found that does not deviate more from the reference trajectory than is allowed by the safety constraint. The corresponding performance functions can be seen in~\cref{fig:context_circle}.}
\label{fig:circle_trajectories_context}
\end{figure}

%!TEX root = ../root.tex

\section{Conclusion and Future Work}
\label{sec:conclusion}

We presented a generalization of the Safe Bayesian Optimization algorithm of~\cite{Sui2015Safe} that allows multiple, separate safety constraints to be specified and applied it to nonlinear control problems on a quadrotor vehicle. Overall, the algorithm enabled efficient and automatic optimization of parameters without violating the safety constraints, which would lead to system failures. Currently, the algorithm is mostly applicable to low-dimensional problems due to the computational burdon of optimizing~\cref{math:safeopt_selection_criterion} and the statistical problem of defining suitable GP priors in high-dimensions. While interesting progress has been made in this direction in the standard Bayesian optimization case, future work could explore this in the safety-critical case.

\begin{acks}
This research was supported in part by SNSF grant {200020\_159557}, NSERC grant {RGPIN-2014-04634}, and the Connaught New Researcher Award.
\end{acks}

%!TEX root = ../root.tex

\section{Proofs}
\label{sec:proofs}
In this section, we provide the proofs for~\cref{thm:safeopt_optimality_safety,thm:safety_only}.

Since we consider the surrogate function~$h(a, i)$ in~\cref{math:surrogate}, we obtain~$q+1$ individual measurements at each iteration, each with individual noise. A measure of how difficult it is to learn an accurate GP model of a function is given by the information capacity. This corresponds to the maximum amount of mutual information between a scalar function~$h$ and measurements~$\hat{\mb{h}}_\mathcal{D}$ at a set of parameters $\mathcal{D}$ of size~$n$. The measurements~$\hat{\mb{h}}_\mathcal{D}$ are corrupted by zero-mean, Gaussian noise. The information capacity is then defined as
\begin{equation}
\gamma_n := \max_{\mathcal{D} \subset \mathcal{A} \times \mathcal{I}, |\mathcal{D}|=n} \mathrm{I}(\hat{\mb{h}}_\mathcal{D}; h),
\end{equation}
which is the maximum amount of information we can obtain about the function~$h$ from~$n$ measurements. The information gain is known to be sublinear in~$n$ for many commonly used kernels~\cite{Srinivas2012Gaussian}. Intuitively, the first samples for the GP model provide a lot of information, since each sample improves the prior significantly. After some iterations the domain~$\mathcal{A}$ is covered with samples in~$\mathcal{D}$, so that additional samples are more correlated with previous data points in~$\mathcal{D}$, rendering the samples less informative. The more prior information we encode in the GP prior, the less information can be gained from the same number of samples.

In our setting, we obtain~${|\mathcal{I}|={q+1}}$ measurements at every iteration step~$n$, each with different, independent noise. The mutual information with regards to these multiple measurements at parameters~$\bar{\mathcal{A}} \subset \mathcal{A}$ can be bounded with
\begin{align}
  \mathrm{I}( \hat{\mb{h}}_{\bar{\mathcal{A}} \times \mathcal{I}} ; h )  &\leq \max_{\bar{\mathcal{A}} \subset \mathcal{A}, |\bar{\mathcal{A}}| \leq n} \mathrm{I}(\hat{\mb{h}}_{\bar{\mathcal{A}} \times \mathcal{I}}  ; h ), \\
&\leq \max_{\mathcal{D} \subset \mathcal{A} \times \mathcal{I}, |\mathcal{D}| \leq n|\mathcal{I}|} \mathrm{I}(\hat{\mb{h}}_{\mathcal{D} }  ; h ), \\
&= \gamma_{|\mathcal{I}|n},
\label{eq:mutual_info_multi_meas_bound}
\end{align}
where~$\bar{\mathcal{A}} \times \mathcal{I}$ is the cartesian product that means we obtain one measurement for every function indexed by~$i \in \mathcal{I}$ at each parameter in~$\bar{\mathcal{A}}$. The first inequality bounds the mutual information gained by~\cref{alg:safe_opt} by the worst-case mutual information, while the second inequality bounds this again by the worst-case mutual information when optimizing over the~$|\mathcal{I}|$ measurements at each iteration step individually. Intuitively, obtaining multiple optimal samples does not fundamentally change the properties of the information gain, but accelerates the rate at which information can be obtained in the worst case by~$|\mathcal{I}|$.

In the following, we assume that~$h(\mb{a}, i)$ has bounded RKHS norm. \cref{thm:confidence_interval}~provides requirements for~$\beta_n$, which will be used in the following to prove the results.

\confidencethm*
\begin{proof}
  Directly follows from~\citet{Chowdhury2017Kernelized}. The only difference is that we obtain~$|\mathcal{I}|$ measurements at every iteration, which causes the information capacity~$\gamma$ to grow at a faster rate. \qed
\end{proof}
\paragraph{Note} Where needed in the following lemmas, we implicitly assume that the assumptions of \cref{thm:confidence_interval} hold, and that~$\beta_n$ is defined as above.

\begin{corollary}
  \label{cor:rkhs}
  For~$\beta_n$ as above, the following holds with probability at least~$1-\delta$:
  \begin{align*}
  \forall n \geq 1,\, \forall i \in \mathcal{I},\, \forall \mb{a} \in \mathcal{A},\, h(\mb{a}, i) \in C_n(\mb{a}, i).
  \end{align*}
\end{corollary}
\begin{proof}
From~\cref{thm:confidence_interval} we know that the true functions are contained in~$Q_n(\mb{a}, i)$ for all iterations~$n$ with probability at least~$1-\delta$. As a consequence, the true functions will be contained in the intersection of these sets with the same probability. \qed
\end{proof}

% \begin{lemma}
% \label{thm:confidence_interval}
%   Assume that~$\|p\|^2_k \leq B$ and~$n_n \leq \sigma,\ \forall n\geq 1$.
%   If~$\beta_n = 2B + 300\gamma_{|\mathcal{I}|n}\log^3(|\mathcal{I}|n/\delta)$, then the following holds with probability at least~$1-\delta$ for all~$i \in \mathcal{I}$:
%   \begin{align*}
%   \forall n \geq 1\,\forall \mb{a} \in \mathcal{A},\ |h(\mb{a},i) - \mu_{n-1}(\mb{a},i)| \leq \beta_n^{1/2}\sigma_{n-1}(\mb{a}, i).
%   \end{align*}
% \end{lemma}
% \begin{proof}
%   See Theorem 6 by \cite{Srinivas2012Gaussian}. \qed
% \end{proof}

\cref{cor:rkhs} gives a choice of~$\beta_n$, which ensures that all the function values of~$h$ are contained within their respective confidence intervals with high probability. In the remainder of the paper, we follow the outline of the proofs in~\cite{Sui2015Safe}, but extended them to account for multiple constraints.

We start by showing the dynamics of important sets and functions. Most importantly, the upper confidence bounds are decreasing, lower confidence bounds increasing with the number of iterations, since the sets~$C_{n+1} \subseteq C_n$ for all iterations~$n$.
\begin{lemma}
  \label{lem:basics}
  The following hold for any~$n \geq 1$:
  \begin{enumerate}[label=(\roman*)]
    \item~$\forall \mb{a} \in \mathcal{A}, \forall i \in \mathcal{I},\, u^i_{n+1}(\mb{a}) \leq u^i_n(\mb{a})$,
    \item~$\forall \mb{a} \in \mathcal{A}, \forall i \in \mathcal{I},\, l^i_{n+1}(\mb{a}) \geq l^i_n(\mb{a})$,
    \item~$\forall \mb{a} \in \mathcal{A}, \forall i \in \mathcal{I},\, w_{n+1}(\mb{a}, i) \leq w_n(\mb{a}, i)$,
    \item~$S_{n+1} \supseteq S_n \supseteq S_0$,
    \item~$S \subseteq R \Rightarrow \Reps(S) \subseteq \Reps(R)$,
    \item~$S \subseteq R \Rightarrow \Rbeps(S) \subseteq \Rbeps(R)$.
  \end{enumerate}
\end{lemma}
\begin{proof}
  (i), (ii), and (iii) follow directly from their definitions and the definition of~$C_n(\mb{a})$.
  \begin{enumerate}[label=(\roman*)]
    \setcounter{enumi}{3}
    \item Proof by induction. Consider the initial safe set,~$S_0$. By definition of~$C_0$ we have for all~$\mb{a} \in S_0$ and~$i \in \mathcal{I}$ that
        \begin{align*}
          l^i_1(\mb{a}) - L \|\mb{a} - \mb{a}\| = l^i_1(\mb{a}) \geq l^i_0(\mb{a}) \geq 0.
        \end{align*}
        It then follows from the definition of~$S_n$ that~${\mb{a} \in S_1}$.

      For the induction step, assume that for some~$n \geq 2$,~$S_{n-1} \subseteq S_n$ and let~${\mb{a} \in S_n}$.
      This means that for all ${i \in \mathcal{I}_g}$, ${ \exists \mb{z}_i \in S_{n-1}, l^i_n(\mb{z}_i) - L \| \mb{z}_i - \mb{a} \| \geq 0}$  by the definition of the safe set.
      But, since~${S_{n-1} \subseteq S_n}$, this implies that~${ \mb{z}_i \in S_n}$,~${\forall i \in \mathcal{I}_g }$.
      Furthermore, by part (ii),~$l^i_{n+1}(\mb{z}) \geq l^i_n(\mb{z}_i)$.
      Therefore, we conclude that for all~${i \in \mathcal{I}_g}$,~${l^i_{n+1}(\mb{z}_i) - L \| \mb{z}_i - \mb{a} \| \geq 0}$, which implies that~$\mb{a} \in S_{n+1}$.
      \item Let~$\mb{a} \in \Reps(S)$. Then, by definition, for all~$i \in \mathcal{I}_g$,~${\exists \mb{z}_i \in S, g_i(\mb{z}_i) - L \|\mb{z}_i - \mb{a} \| \geq 0}$.
        But, since~$S \subseteq R$, it means that~$\mb{z}_i \in R \forall i \in \mathcal{I}_g$, and, therefore,~$g_i(\mb{z}_i) - L \| \mb{z}_i - \mb{a} \| \geq 0$ for all~$i \in \mathcal{I}_g$ also implies that~$\mb{a} \in \Reps(R)$.
      \item This follows directly by repeatedly applying the result of part (v).
  \end{enumerate} \qed
\end{proof}

Using the previous results, we start by showing that, after a finite number of iterations, the safe set has to expand if possible. As a first step, note that the set of expanders and maximizers are contained in each other as well if the safe set does not increase:

\begin{lemma}
  \label{lem:gtmt}
  For any~$n_1 \geq n_0 \geq 1$, if~$S_{n_1} = S_{n_0}$, then, for any~$n$, such that~$n_0 \leq n < n_1$, it holds that
  \begin{align*}
    G_{n+1} \cup M_{n+1} \subseteq G_n \cup M_n.
  \end{align*}
\end{lemma}
\begin{proof}
  Given the assumption that~$S_n$ does not change, both~$G_{n+1} \subseteq G_n$ and~$M_{n+1} \subseteq M_n$ follow directly from the definitions of~$G_n$ and~$M_n$.
  In particular, for~$G_n$, note that for any~$\mb{a} \in S_n$,~$e^i_n(\mb{a})$ is decreasing in~$n$ for all~$i \in \mathcal{I}_g$, since~$u^i_n(\mb{a})$ are decreasing in~$n$.
  For~$M_n$, note that~$\max_{\mb{a}' \in S_n}l^f_n(\mb{a}')$ is increasing in~$n$, while~$u^f_n(\mb{a})$ is decreasing in~$n$ (see \cref{lem:basics} (i), (ii)). \qed
\end{proof}

When running the~\textsc{SafeOpt-MC} algorithm, we repeatedly choose the most uncertain element from~$G_n$ and~$M_n$. Since these sets are contained in each other if the safe set does not expand, we gain more information about these sets with each sample. Since the information gain is bounded, this allows us to bound the uncertainty in terms of the information gain over the entire set:

\begin{lemma}
  \label{lem:wt}
  For any~$n_1 \geq n_0 \geq 1$, if~$S_{n_1} = S_{n_0}$ and~$C_1 \defeq 8 / \log(1 + \sigma^{-2})$, then, for any~$n$, such that~$n_0 \leq t \leq n_1$, it holds for all~$i \in \mathcal{I}$ that
    \begin{align*}
      w_n(\mb{a}_n, i) \leq \sqrt{\frac{C_1 \beta_n \gamma_{|\mathcal{I}|n}}{n - n_0}}.
    \end{align*}
\end{lemma}
\begin{proof}
  Given \cref{lem:gtmt}, the definition of~$\mb{a}_n \defeq \argmax_{\mb{a} \in G_n \cup M_n}(w_n(\mb{a}))$, and the fact that, $w^i_n(\mb{a}_n) \leq 2\beta^{1/2}_n \max_{\in \in \mathcal{I}} \sigma_{n-1}(\mb{a}_n, i) = 2\beta^{1/2}_n(\mb{a}_n, i_n)$, the proof is completely analogous to that of Lemma 5.3 by \cite{Srinivas2012Gaussian}. We only highlight the main differences here, which results from having several functions.
  \begin{equation}
    w^i_n(\mb{a}_n) \leq 2\beta^{1/2}_n \max_{\in \in \mathcal{I}}\sigma_{n-1}(\mb{a}_n, i),
  \end{equation}
  which following~\cite[Lemma 5.4]{Srinivas2012Gaussian} leads to
  \begin{equation}
  \sum_{j=1}^n w^2_j(\mb{a}_j, i_j) \leq \beta_{|\mathcal{I}| n}^{1/2} \mathrm{I}(\hat{\mb{h}}_{\bar{\mathcal{D}_n}}; h ),
\end{equation}
where~$\bar{\mathcal{D}_n} = \{\mb{a}_n, i_n \}$. Now using monotonicity of the mutual information, we have that
\begin{align}
\sum_{j=1}^n w^2_j(\mb{a}_j, i_j) &\leq C_1 \beta_{|\mathcal{I}| n}^{1/2} \mathrm{I}(\hat{\mb{h}}_{\mathcal{D}_n \times \mathcal{I}}; h ), \\
&\leq C_1 \beta_{|\mathcal{I}| n}^{1/2} \gamma_{|\mathcal{I}|n}
\end{align}
by~\cref{eq:mutual_info_multi_meas_bound}. \qed
\end{proof}

\begin{corollary}
  \label{cor:wt}
  For any~$n \geq 1$, if~$C_1$ is defined as above,~$\N$ is the smallest positive integer satisfying~$\displaystyle\frac{\N}{\beta_{n+\N} \gamma_{|\mathcal{I}|(n+\N)}} \geq \frac{C_1}{\epsilon^2}$, and~$S_{n+\N} = S_n$, then, for any~$\mb{a} \in G_{n+\N} \cup M_{n+\N}$, and for all~$i \in \mathcal{I}$ it holds that
  \begin{align*}
    w_{n+\N}(\mb{a}, i) \leq \epsilon.
  \end{align*}
\end{corollary}

\paragraph{Note} Where needed in the following lemmas, we assume that~$C_1$ and~$\N$ are defined as above.

That is, after a finite number of evaluations~$N_n$ the most uncertain element within these sets is at most~$\epsilon$. Given that the reachability operator in~\cref{math:baseline_reachable_set} is defined in terms of the same accuracy, it allows us to show that after at most~$N_n$ evaluations, the safe set has to increase unless it is impossible to do so:

\begin{lemma}
  \label{lem:expansion0}
  For any~$n \geq 1$, if~$\Rbeps(S_0) \setminus S_n \not= \varnothing$, then~$\Reps(S_n) \setminus S_n \not= \varnothing$.
\end{lemma}
\begin{proof}
  Assume, to the contrary, that~$\Reps(S_n) \setminus S_n = \varnothing$.
  By definition,~$\Reps(S_n) \supseteq S_n$, therefore~$\Reps(S_n) = S_n$.
  Iteratively applying~$\Reps$ to both sides, we get in the limit~$\Rbeps(S_n) = S_n$.
  But then, by \cref{lem:basics} (iv) and (vi), we get
  \begin{align}
    \label{eq:expansion0}
    \Rbeps(S_0) \subseteq \Rbeps(S_n) = S_n,
  \end{align}
  which contradicts the lemma's assumption that~$\Rbeps(S_0) \setminus S_n \not= \varnothing$. \qed
\end{proof}

\begin{lemma}
  \label{lem:expansion}
  For any~$n \geq 1$, if~$\Rbeps(S_0) \setminus S_n \not= \varnothing$, then the following holds with probability at least~$1-\delta$:
  \begin{align*}
    S_{n+\N} \supsetneq S_n.
  \end{align*}
\end{lemma}
\begin{proof}
  By \cref{lem:expansion0}, we get that,~$\Reps(S_n) \setminus S_n \not= \varnothing$, Equivalently, by definition, for all~$i \in \mathcal{I}_g$
    \begin{align}
      \label{eq:expansion1}
      \exists \mb{a} \in \Reps(S_n) \setminus S_n,\,\exists \mb{z}_i \in S_n \colon g_i(\mb{z}_i) - \epsilon - L \| \mb{z}_i - \mb{a} \| \geq 0.
    \end{align}

  Now, assume, to the contrary, that~$S_{n+\N} = S_n$ (see \cref{lem:basics} (iv)), which implies that~$\mb{a} \in \mathcal{A} \setminus S_{n+\N}$ and~$\mb{z}_I \in S_{n+\N} \forall i \in \mathcal{I}_g$.
  Then, we have for all~$i \in \mathcal{I}_g$
  \begin{align*}
    u^i_{n+\N}(\mb{z}_i) - L \| \mb{z}_i -  \mb{a} \| &\geq g_i(\mb{z}_i) - L \| \mb{z} - \mb{a} \| \tag*{by \cref{thm:confidence_interval}}\\
    &\geq g_i(\mb{z}_i) - \epsilon - L \| \mb{z} - \mb{a} \| \\
    &\geq 0. \tag*{by \eqref{eq:expansion1}}
  \end{align*}
  Therefore, by definition,~$e_{n+\N}(\mb{z}_i) > 0$, which implies~$\mb{z}_i \in G_{n+\N},\, \forall i \in \mathcal{I}_g$.

  Finally, since~$S_{n+\N} = S_n$ and~$\mb{z}_i \in G_{n+\N} \forall i \in \mathcal{I}_g$, we know that for all~$i \in \mathcal{I},\, w_{n + \N}(\mb{a}', i) \leq \epsilon$. (\cref{cor:wt}). Hence, for all~$i \in \mathcal{I}_g$,
  \begin{align*}
    l^i_{n+\N}(\mb{z}_i) - L \|\mb{z}_i - \mb{a} \|
    &\geq g_i(\mb{z}_i) - w(\mb{z}_i, i) - L \|a - \mb{z}_i\| \tag*{by~\cref{thm:confidence_interval}} \\
    &\geq g_i(\mb{z}) - \epsilon - L \|a - \mb{z}_i \| \tag*{by \cref{cor:wt}}\\
    &\geq 0 \tag*{by \cref{eq:expansion1}}.
  \end{align*}
  This means we get~$\mb{a} \in S_{n+\N}$, which is a contradiction. \qed
\end{proof}

Intuitively, repeatedly applying the previous result leads to full safe exploration within a finite domain~$\mathcal{A}$. In particular, it follows that if~$S_{n+\N} = S_n$, then the safely reachable set has been fully explored to the desired accuracy. From this it follows, that the pessimistic estimate in~\cref{math:safeopt_pessimistic_estimate} is also $\epsilon$-close to the optimum value within the safely reachable set,~$\Rbeps(S_0)$:

\begin{lemma}
  \label{lem:max}
  For any~$n \geq 1$, if~$S_{n+\N} = S_n$, then the following holds with probability at least~$1-\delta$:
  \begin{align*}
    f(\mb{a}_{n+\N}) \geq \displaystyle\max_{\mb{a} \in \Rbeps(S_0)} f(\mb{a}) - \epsilon.
  \end{align*}
\end{lemma}
\begin{proof}
  Let~$\mb{a}^* \defeq \argmax_{\mb{a} \in S_{n+\N}} f(\mb{a})$.
  Note that~$\mb{a}^* \in M_{n+\N}$, since
  \begin{align*}
    u^f_{n+\N}(\mb{a}^*) &\geq f(\mb{a}^*) \tag*{by \cref{thm:confidence_interval}}\\
    &\geq f(\mb{a}) \tag*{by definition of~$\mb{a}^*$}\\
    &\geq l^f_{n+\N}(\mb{a}) \tag*{by \cref{thm:confidence_interval}}\\
    &\geq \max_{\mb{a} \in S_{n+\N}} l^f_{n+\N}(\mb{a}). \tag*{by definition of~$\mb{a}$}\\
  \end{align*}

  We will first show that~$f(\mb{a}_{n+\N}) \geq f(\mb{a}^*) - \epsilon$.
  Assume, to the contrary, that
  \begin{align}
    \label{eq:max1}
    f(\mb{a}_{n+\N}) < f(\mb{a}^*) - \epsilon.
  \end{align}
  Then, we have
  \begin{align*}
    l^f_{n+\N}(\mb{a}^*) &\leq l^f_{n+\N}(\mb{a}) \tag*{by definition of~$\mb{a}$}\\
    &\leq f(\mb{a}) \tag*{by \cref{thm:confidence_interval}}\\
    &< f(\mb{a}^*) - \epsilon \tag*{by \eqref{eq:max1}}\\
    &\leq u^f_{n+\N}(\mb{a}^*) - \epsilon \tag*{by \cref{thm:confidence_interval}}\\
    &\leq l^f_{n+\N}(\mb{a}^*), \tag*{by \cref{cor:wt} and~$\mb{a}^* \in M_{n+\N}$}
  \end{align*}
  which is a contradiction.

  Finally, since~$S_{n+\N} = S_n$, \cref{lem:expansion} implies that~$\Rbeps(S_0) \subseteq S_n = S_{n+\N}$.
  Therefore,
  \begin{align*}
    \max_{\mb{a} \in \Rbeps(S_0)} f(\mb{a}) - \epsilon &\leq \max_{\mb{a} \in S_{n+\N}} f(\mb{a}) - \epsilon \tag*{$\Rbeps(S_0) \subseteq S_{n+\N}$}\\
    &= f(\mb{a}^*) - \epsilon \tag*{by definition of~$\mb{a}^*$}\\
    &\leq f(\mb{a}_{n+\N}) \tag*{proven above}.
  \end{align*} \qed
\end{proof}

\begin{corollary}
  \label{cor:max}
  For any~$n \geq 1$, if~$S_{n+\N} = S_n$, then the following holds with probability at least~$1-\delta$:
  \begin{align*}
    \forall n' \geq 0, f(\mb{a}_{n+\N+n'}) \geq \displaystyle\max_{\mb{a} \in \Rbeps(S_0)} f(\mb{a}) - \epsilon.
  \end{align*}
\end{corollary}
\begin{proof}
  This is a direct consequence of the proof of the preceding lemma, combined with the facts that both~$S_{n+\N+n'}$ and~$l^f_{n+\N+n'}(\mb{a}_{n+\N+n'})$ are increasing in~$n'$ (by \cref{lem:basics} (iv) and (ii) respectively), which imply that~$\max_{\mb{a} \in S_{n+\N+n'}} l^f_{n+\N+n'}(\mb{a})$ can only increase in~$n'$. \qed
\end{proof}

Moreover, since we know the true function is contained within the confidence intervals, we cannot go beyond the safe set if we knew the function perfectly everywhere,~$\Rbo$:

\begin{lemma}
  \label{lem:stleq}
  For any~$n \geq 0$, the following holds with probability at least~$1-\delta$:
  \begin{align*}
    S_n \subseteq \Rbo(S_0).
  \end{align*}
\end{lemma}
\begin{proof}
  Proof by induction. For the base case,~$n = 0$, we have by definition that~$S_0 \subseteq \Rbo(S_0)$.

  For the induction step, assume that for some~$n \geq 1$,~$S_{n-1} \subseteq \Rbo(S_0)$.
  Let~$\mb{a} \in S_n$, which, by definition, means that for all~$i \in \mathcal{I}_g$~$\exists \mb{z}_i \in S_{n-1}$, such that
  \begin{align*}
    & l^i_n(\mb{z}_i) - L \| \mb{z}_i - \mb{a} \| \geq 0\\
    \Rightarrow\ \ & g_i(\mb{z}_i) - L \| \mb{z}_i - \mb{a} \| \geq 0. \tag*{by \cref{thm:confidence_interval}}\\
  \end{align*}
  Then, by definition of~$\Rbo$ and the fact that~$\mb{z}_i \in \Rbo(S_0)$ for all~$i \in \mathcal{I}_g$, it follows that~$\mb{a} \in \Rbo(S_0)$. \qed
\end{proof}

The previous results is enough to show that we eventually explore the full safe set by repeatedly applying~\cref{lem:expansion}:

\begin{lemma}
  \label{lem:existence}
  Let~$n^*$ be the smallest integer, such that~$n^* \geq |\Rbo(S_0)|T_{n^*}$.
  Then, there exists~$n_0 \leq n^*$, such that~$S_{n_0+T_{n_0}} = S_{n_0}$.
\end{lemma}
\begin{proof}
  Assume, to the contrary, that for any~$n \leq n^*$,~$S_n \subsetneq S_{n+T_n}$. (By \cref{lem:basics} (iv), we know that~$S_n \subseteq S_{n+T_n}$.)
  Since~$\N$ is increasing in~$n$, we have
  \begin{align*}
    S_0 \subsetneq S_{n_0} \subseteq S_{T_{n^*}} \subsetneq S_{T_{n^*}+T_{T_{n^*}}} \subseteq S_{2T_{n^*}} \subsetneq \cdots,
  \end{align*}
  which implies that, for any~$0 \leq k \leq |\Rbo(S_0)|$, it holds that~$|S_{kT_{n^*}}| > k$.
  In particular, for~$k^* \defeq |\Rbo(S_0)|$, we get
  \begin{align*}
    |S_{k^*T}| > |\Rbo(S_0)|
  \end{align*}
  which contradicts~$S_{k^*T} \subseteq \Rbo(S_0)$ by \cref{lem:stleq}. \qed
\end{proof}

\begin{corollary}
  \label{cor:existence}
  Let~$n^*$ be the smallest integer, such that~$\displaystyle\frac{n^*}{\beta_{n^*}\gamma_{|\mathcal{I}|n^*}} \geq \frac{C_1|\Rbo(S_0)|}{\epsilon^2}$.
  Then, there exists~$n_0 \leq n^*$, such that~$S_{n_0+T_{n_0}} = S_{n_0}$.
\end{corollary}
\begin{proof}
This is a direct consequence of combining \cref{lem:existence} and \cref{cor:wt}. \qed
\end{proof}

Since we showed that we completely explore the safe set and that we remain safe throughout the exploration procedure, we are ready to state the main results:

\begin{lemma}
  \label{lem:safe}
  If~$h$ is \LLC, then, for any~$n \geq 0$, the following holds with probability at least~$1-\delta$ for all~${i \in \mathcal{I}_g}$:
  \begin{align*}
    \forall \mb{a} \in S_n, g_i(\mb{a}) \geq 0.
  \end{align*}
\end{lemma}
\begin{proof}
  We will prove this by induction.
  For the base case~$n = 0$, by definition, for any~$\mb{a} \in S_0$ and~$i \in \mathcal{I}_g$,~$g_i(\mb{a}) \geq 0$.

  For the induction step, assume that for some~$n \geq 1$, for any~$\mb{a} \in S_{n-1}$ and for all~$i \in \mathcal{I}_g$,~$g_i(\mb{a}) \geq 0$.
  Then, for any~$\mb{a} \in S_n$, by definition, for all~$i \in \mathcal{I}_g$,~$\exists \mb{z}_i \in S_{n-1}$,
  \begin{align*}
    0 &\leq l^i_n(\mb{z}_i) - L \| \mb{z}_i - \mb{a} \|\\
      &\leq g_i(\mb{z}_i) - L \|\mb{z}_i - \mb{a}\| \tag*{by \cref{thm:confidence_interval}}\\
      &\leq g_i(\mb{a}) \tag*{by~$L$-Lipschitz-continuity}.
  \end{align*} \qed
\end{proof}

\safeopttheorem*
\begin{proof}
  The first part of the theorem is a direct consequence of \cref{lem:safe}.
  The second part follows from combining \cref{cor:max} and \cref{cor:existence}. \qed
\end{proof}

\bibliographystyle{SageH}
\bibliography{root}

\end{document}